\documentclass[twoside]{article}

\usepackage[accepted]{aistats2021}
\usepackage[utf8]{inputenc} 
\usepackage[T1]{fontenc}    
\usepackage{hyperref}       
\usepackage{url}            
\usepackage{booktabs}       
\usepackage{amsfonts}       
\usepackage{nicefrac}       
\usepackage{microtype}      

\usepackage{url}
\usepackage{algorithm}
\usepackage{algorithmic}
\usepackage{graphicx}
\usepackage{amsthm}
\usepackage{amssymb}
\usepackage{multicol}
\usepackage{multirow}
\usepackage{amsmath}

\usepackage{array}
\usepackage{afterpage}
\usepackage{amsthm}
\newtheorem{theorem}{Theorem}[section]
\newtheorem{corollary}{Corollary}[theorem]
\newtheorem{lemma}[theorem]{Lemma}

\usepackage{wrapfig}
\usepackage{booktabs}

\makeatletter
\newcommand{\thickhline}{%
    \noalign {\ifnum 0=`}\fi \hrule height 1pt
    \futurelet \reserved@a \@xhline
}

\newcolumntype{"}{@{\hskip\tabcolsep\vrule width 1pt\hskip\tabcolsep}}
\makeatother
\usepackage{amsmath,amssymb}

\DeclareMathOperator{\EX}{\mathbb{E}}
\newcommand\norm[1]{\left\lVert#1\right\rVert}

\usepackage{amsmath,amsfonts,bm}
\usepackage{latexsym} 
\usepackage{booktabs}
\usepackage{amsmath}
\usepackage{verbatim}
\usepackage{amssymb}
\usepackage{amsthm}
\usepackage{algorithm}
\usepackage{algorithmic}
\usepackage{graphicx}
\usepackage{subfigure}
\usepackage{bbm}

\usepackage{booktabs}       
\usepackage{amsfonts}       
\usepackage{nicefrac}       
\usepackage{microtype}      

\def\eqref#1{equation~\ref{#1}}

\def\1{\bm{1}}

\DeclareMathAlphabet{\mathsfit}{\encodingdefault}{\sfdefault}{m}{sl}
\SetMathAlphabet{\mathsfit}{bold}{\encodingdefault}{\sfdefault}{bx}{n}

\usepackage[small]{caption}

\begin{document}

%

%

\twocolumn[

\aistatstitle{Adaptive Geo-Topological Independence Criterion}

%


\aistatsauthor{ Baihan Lin \And Nikolaus Kriegeskorte }

\aistatsaddress{ Columbia University \And Columbia University  } 

]

\begin{abstract}

Testing two potentially multivariate variables for statistical dependence on the basis finite samples is a fundamental statistical challenge. Here we explore a family of tests that adapt to the complexity of the relationship between the variables, promising robust power across scenarios. Building on the distance correlation, we introduce a family of adaptive independence criteria based on nonlinear monotonic transformations of distances. We show that these criteria, like the distance correlation and RKHS-based criteria, provide dependence indicators. We propose a class of adaptive (multi-threshold) test statistics, which form the basis for permutation tests. These tests empirically outperform many established tests in average and worst-case statistical sensitivity across a range of univariate and multivariate relationships, offer useful insights to the data and may deserve further exploration.\footnote{Correspondence to {\tt\{bl2681, nk2765\}@columbia.edu}. Code at \href{https://github.com/doerlbh/AGTIC}{\underline{https://github.com/doerlbh/AGTIC}}.}


\end{abstract}
\section{Introduction}
\label{sec:intro}

Detecting statistical dependence between random variables is a fundamental problem of statistics. The simplest scenario is detecting linear or monotonic univariate relationships, where Pearson's r, Spearman's $\rho$, or Kendall's $\tau$ can serve as test statistics. Often researchers need to detect nonlinear relationships between multivariate variables. In recent years, many nonlinear statistical dependence indicators have been developed: distance-based methods such as distance or Brownian correlation (dCor) \cite{szekely2007measuring,szekely2009brownian}, mutual information (MI)-based methods with different estimators \cite{kraskov2004estimating, pal2010estimation,steuer2002mutual}, kernel-based methods such as the Hilbert-Schmidt Independence Criterion (HSIC) \cite{gretton2005measuring, gretton2008kernel} and Finite Set Independence Criterion (FSIC) \cite{jitkrittum2016adaptive}, and other dependence measures including Maximal Information Coefficient (MIC) \cite{reshef2011detecting,reshef2013equitability}, Multiscale Graph Correlation (MGC) \cite{vogelstein2019discovering} and HHG’s test (HHG) \cite{heller2013consistent}. 

There's no free lunch: any indicator will outperform any other indicator given data whose dependence structure it is better suited to detect. However, it is desirable to develop indicators that adapt to the grain of the dependency structure and to the amount of data available to maintain robust power across relationships found in real applications. Except for FSIC, the established methods are not adaptive. Some of them are sensitive to the setting of hyperparameters, or have low statistical power for detecting important nonlinear or high-dimensional relationships \cite{simon2014comment}. 


Here we propose a family of adaptive distance-based independence criteria inspired by two ideas: (1) Representational geometries can be compared by correlating distance matrices \cite{kriegeskorte2013representational}. (2) We can relax the constraint of linear correlation of the distances by nonlinearly transforming distance matrices, such that they capture primarily neighbor relationships. Such a transformed (e.g. thresholded) distance matrix captures the topology, rather than the geometry. Detecting matching topologies between two spaces $\mathcal{X}$ and $\mathcal{Y}$ will indicate statistical dependency. As illustrated in Fig. \ref{fig:example}, given a specific multivariate associate pattern, the proposed Adaptive Geo-Topological Independence Criterion (AGTIC) transforms the pairwise distances that are too small or too big, only keeping a subset of the original distances as the matching topology. In the presented case, smaller distances are the most distinctive topological edges in the spiral pattern, which is not the case in the linear pattern.

We show analytically that a family of such geo-topological relatedness indicators are 0 (in the limit of infinite data) if and only if multivariate variables $X$ and $Y$ are statistically independent. The geo-topological indicators are based on the distance correlation, computed after a parameterized monotonic transformation of the distance matrices for spaces $\mathcal{X}$ and $\mathcal{Y}$. We use an adaptive search framework to automatically select the parameters of the monotonic transform so as to maximize the distance correlation. We show that monotonic nonlinear operators like the proposed geo-topological transformation belong to a separable space that can be understood as an RKHS-based kernel indicator of dependency. The adaptive threshold search renders the dependence test robustly powerful across a wide spectrum of scenarios and across different noise amplitudes and sample sizes, while guaranteeing (via permutation test) the specificity at a false positive rate of 5\%. 


\section{Geometric and Topological Insights}

\begin{figure}[tb]
\centering
\includegraphics[width=0.235\linewidth]{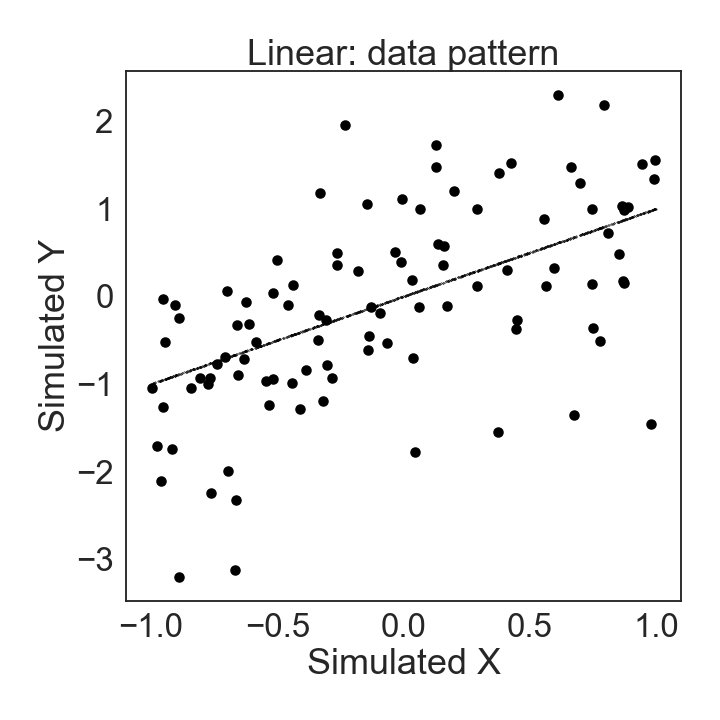}
\includegraphics[width=0.235\linewidth]{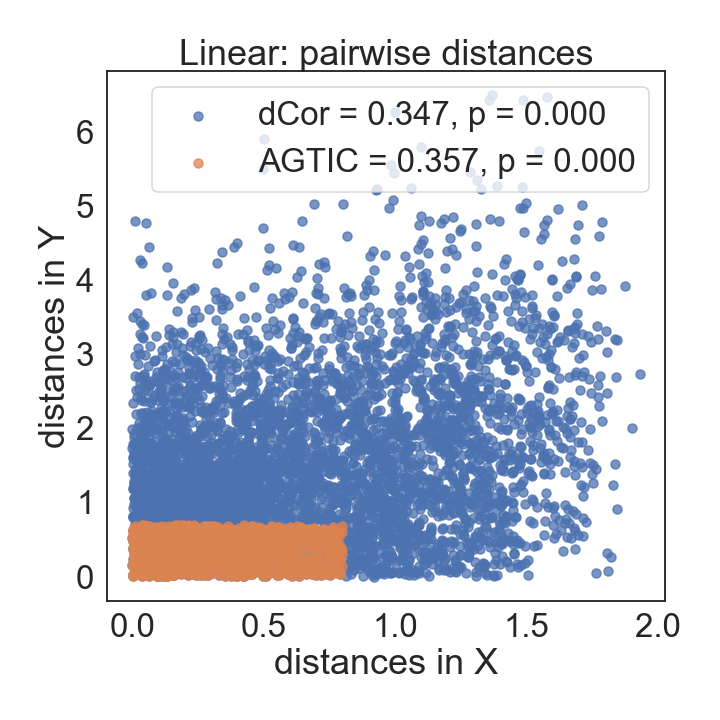}
\includegraphics[width=0.235\linewidth]{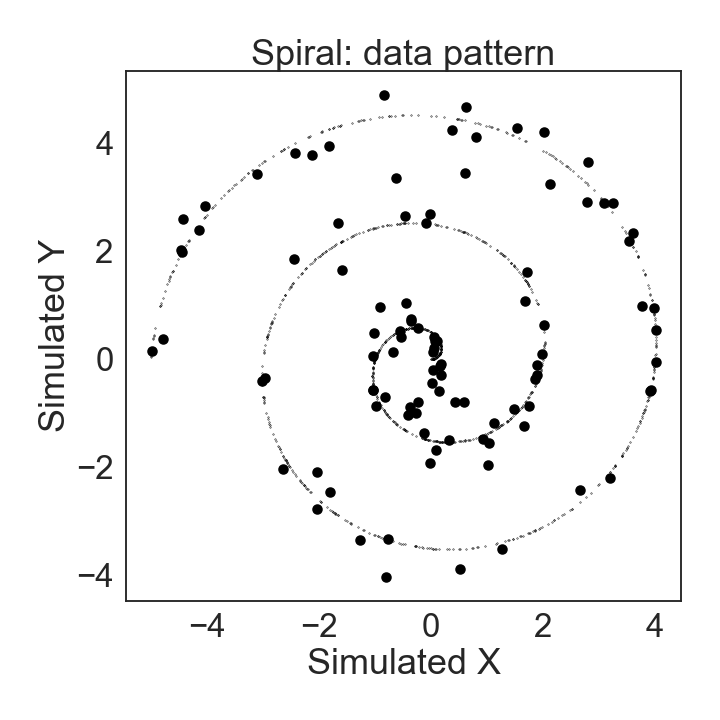}
\includegraphics[width=0.235\linewidth]{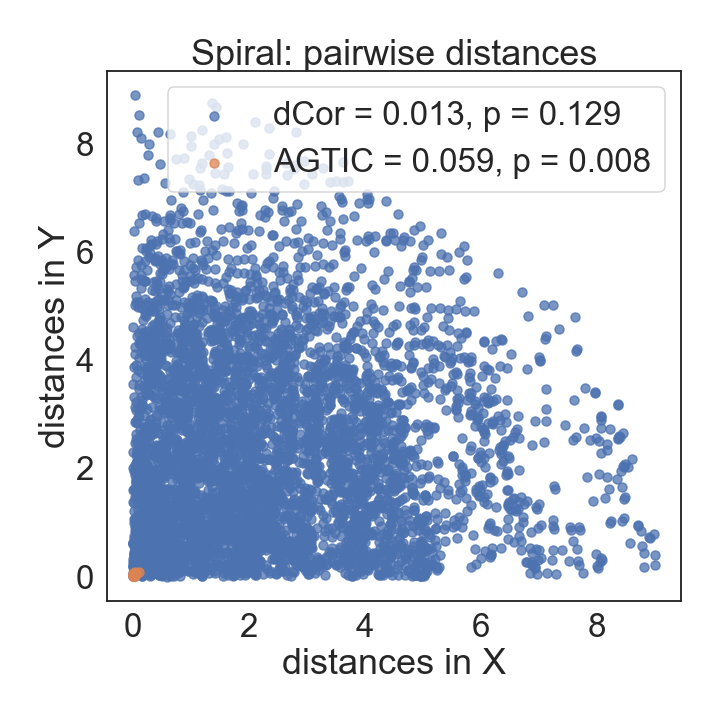}
     \vspace{-0.5em}
\caption{\textbf{Motivation:} dCor used all the original distances (blue dots) to compute its test statistic and p-value, whereas AGTIC transforms the smaller distances to zero, the larger distances to the maximum, and the rest (orange dots) to be distinctive from one another in the original scale. By reshaping the original distances and emphasizing certain distances, AGTIC discovered the dependency in the spiral pattern that dCor missed.}\label{fig:example}
     \vspace{-1em}
\end{figure}

As in the introduction, the multivariate relationship we wish to investigate here is high-dimensional, noisy, nonlinear, and most likely non-stationary. Neuroscience research, for instance, generates high-dimensional functional magnetic resonance imaging (fMRI) with systematic noises from head movement, the heart beats, or breathing, in different textures across subjects. The specialized application of independence testing in cognitive neuroscience is the Representational similarity analysis (RSA), which aims to find dependence patterns within brain-activity measurement, behavioral measurement, and computational models \cite{kriegeskorte2008representational}. Like traditional statistics such as dCor, RSA computed pairwise distances (or dissimilarities) between neural activities among different stimuli (e.g. images shown to a subject). These dissimilary matrices are usually considered as representational \textit{geometry}. 

\textit{Topology}, on the other hand, has different definitions in different contexts. In the context of computational topology, the analysis is often accomplished by topological data analysis (TDA), which is a successful method to discover patterns and meanings in the shape of the data \cite{epstein2011topological}. For example, the persistence homology diagrams can help reveal the most fundamental shape of the data; the Mapper algorithm is able to transform any data (such as point-cloud data) or function output (such as a similarity measure) into a graph which provides a compressed summary of the data distributions and association patterns \cite{singh2007topological}. In the context of our investigation, topology is defined as the consistent dissimilarities that carry multivariate dependence, an abstraction of the representational space independent of systematic noises from data collection procedures, while geometry is defined as the dissimilarity distances. We define the \textit{Geo-Topology}, as the transformed geometry ``denoised'' to capture only the dependence-relevant dissimilarity, to be conceptually considered a topology.  

There are two motivations for this in applications like neural data, one theoretical and one data-analytical. From a theoretical perspective, the computational function of a brain region might depend more on the local than on the global representational geometry, i.e. on differences among small representational distances rather than differences among large representational distances. The local geometry determines, which stimuli the representation renders indiscriminable, which it discriminates, but places together in a cluster, and which it places in different neighborhoods. The global geometry of the clusters (whether two stimuli are far or very far from each other in the representational space) may be less relevant to computation: In a high-dimensional space a set of randomly placed clusters will tend to afford linear separation of arbitrary dichotomies among clusters \cite{kushnir2018} independent of the exact global geometry. Like a storage room, a representational space may need to collocalize related things, while the global location of these categories may be arbitrary.

From a data-analytical perspective, conversely, small distances may be unreliable given the various noise sources that may affect the measurements. From both theoretical and data-analytical perspectives, it seems possible that focusing our sensitivity on a particular range of distances turns out to be advantageous because is reduces the influence of noise and/or arbitrary variability (e.g., of the global geometry) that does not reflect dependency function. In order to suppress these noise, we would like to find a lower distance threshold $l $ below which we consider certain data points as co-localized (i.e., the points have collapsed into the same node in the graph). Between the two thresholds we place a continuous linear transition to retain some geometrical sensitivity in the range where it matters, as in Fig. \ref{fig:RGTA}. This formulation encapsulates the special case of full geometry: one possibility is that the ideal setting is $l =0$, $u=max$, i.e., the original distances.

The novel contribution is the idea to extract the geometry and topology from pairwise distances for independence testing. In the following sections, we introduce AGTIC, an attractive adaptive criterion, the theoretical result that the AGTIC is a proper independence criterion, and the empirical merit of its robust power to detect different kinds of statistical dependency.

\begin{figure}[tb]
\centering
\includegraphics[width=1.0\linewidth]{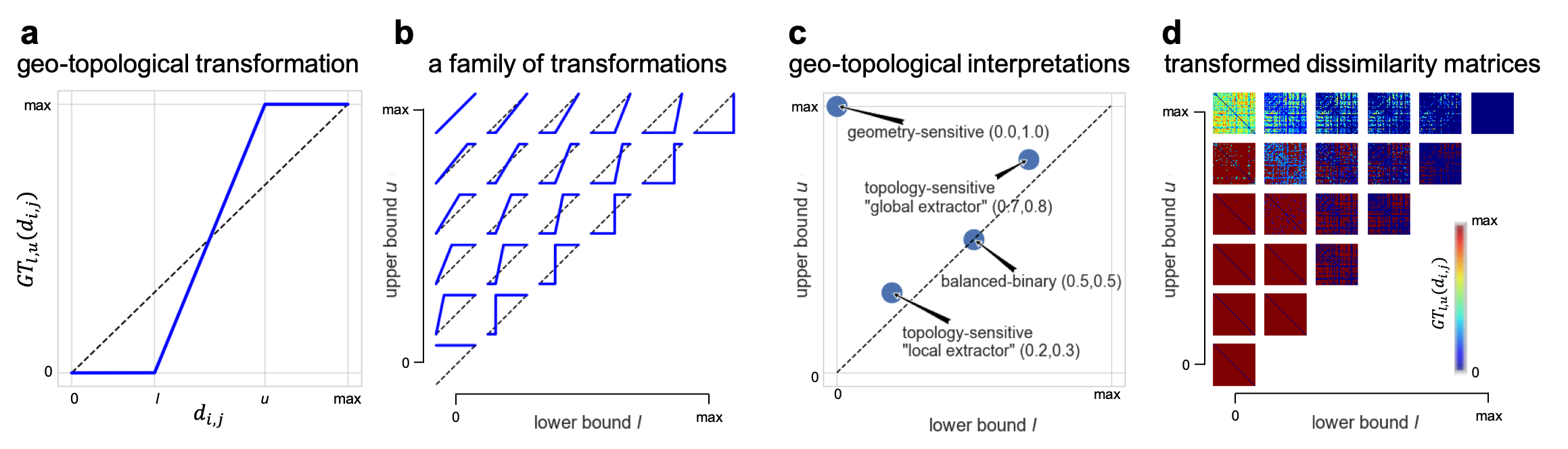}
     \vspace{-0.5em}
\par\caption{\textbf{Geo-Topological (GT) Transforms}. (\textbf{a}) To discard geometrical information that is either not meaningful or unreliable, we subject each distance $d_{i,j}$ for stimuli $i$ and $j$ to a \textit{monotonic transform} $GT_{l,u} (d_{i,j})$. (\textbf{b}) We refer to this family of transforms as geo-topological, because it combines aspects of geometry and topology. (\textbf{c}) Depending on the choice of a lower and upper bound, the transform can threshold ($l =u$) at an arbitrary level, adapting to the granularity of a data set. It can also preserve all ($l =0$, $u=\max$) or some ($l<u$) geometrical information. (\textbf{d}) monotonically transformed distance matrices under a set of threshold pairs.}\label{fig:RGTA}
     \vspace{-0.5em}
\end{figure}
\section{Primers and Definitions of AGTIC}

\textbf{Problem description.} Let $\mathbbm{P}_x$ and $\mathbbm{P}_y$ be the marginal distributions on space $\mathcal{X}$ and $\mathcal{Y}$ and $\mathbbm{P}_{xy}$ be a Borel probability measure defined on their domain $\mathcal{X} \times \mathcal{Y}$. Given the independent and identically distributed (i.i.d.) sample $Z:=(X, Y)=\{(x_1,y_1),\cdots,(x_m,y_m)\}$ of size $m$ drawn independently and identically distributed according to $\mathbbm{P}_{xy}$, with each row corresponding to an observation of both variables, the statistical test $\mathcal{T}(Z): (\mathcal{X} \times \mathcal{Y} \mapsto \{0,1\})$ is used to distinguish between the null hypothesis $\mathcal{H}_0: \mathbbm{P}_{xy}=\mathbbm{P}_x\mathbbm{P}_y$ and the alternative hypothesis $\mathcal{H}_1: \mathbbm{P}_{xy}\not=\mathbbm{P}_x\mathbbm{P}_y$. 

\textbf{Distance correlation (dCor).}
Distance covariance was introduced by \cite{szekely2007measuring} to test dependence between random variables $X$ and $Y$ with finite first moments  in space $\mathcal{X}$ and $\mathcal{Y}$, computed in terms of a weighted $L_2$ norm between the characteristic functions of the joint distribution of $X$ and $Y$ and the product of their marginals, computed in terms of certain expectations of pairwise Euclidean distances: 

\vspace{-1.5em}
\begin{align}
\label{eq:dCov}
\begin{split}
\mathcal{V}^2(X,Y) =& \EX[\norm{X-X'}\norm{Y-Y'}] \\
&+ \EX[\norm{X-X'}]\EX[\norm{Y-Y'}] \\
&- 2\EX[\norm{X-X'}\norm{Y-Y''}]
\end{split}
\end{align}
\vspace{-0.5em}

where $\EX$ denotes expected values, $(X,Y)$ and $(X',Y')$ are drawn i.i.d from $\mathbbm{P}_{xy}$, primed random variables $(X',Y')$ and $(X'',Y'')$ are i.i.d. copies of the variables $X$ and $Y$. Distance correlation (dCor) is obtained by dividing $\mathcal{V}^2(X,Y)$ by the product of their distance standard deviations:

\vspace{-1.5em}
\begin{equation}
\label{eq:dCor}
dCor(X,Y) = \frac{\mathcal{V}^2(X,Y)}{\sqrt{\mathcal{V}^2(X,X)\mathcal{V}^2(Y,Y)}}
\end{equation}
\vspace{-0.5em}

\cite{lyons2013distance} showed that if metrics $\rho \mathcal{X}$ and $\rho \mathcal{Y}$ satisfy strong negative type, the distance correlation in a metric space characterizes independence: $\mathcal{V}^2_{\rho \mathcal{X},\rho\mathcal{Y}}(X,Y)=0 \Leftrightarrow X$ and $Y$ are independent.


\textbf{Geo-Topological transform (GT).} Suppose $d_{i,j}$ is the distance between two sample observation $(x_i,x_j) \overset{i.i.d}\sim X$. Let $GT(d_{i,j})$ be the general form of a nonlinear monotonic transformation parameterized by two positive real numbers $l$ and $u$ satisfying $l<u$. Let $d_{\max} = \max_{i,j\in (1,\cdots,m)} d_{i,j}$ be the largest pairwise distance in space $X$. Here we define $GT_{l,u}(d_{i,j})$ to be the simplified version of the general GT transform as a continuous nonlinear \textit{bounded} \textit{functional} $f(d;l,u,d_{\max})$ onto $L^2[0,1]$:

\vspace{-1.5em}
\begin{equation}
f(d;l,u) =
  \begin{cases}
0 & \text{if $0 \leq d < l$} \\
d_{\max}\cdot\frac{d-l}{u-l} & \text{if $l \leq d < u$} \\
d_{\max} & \text{if $u \leq d$}
  \end{cases}
\end{equation}
\vspace{-0.5em}

as the empirical choice for our test statistics. However, the following theoretical properties also apply to other type of monotonic transforms. Fig. \ref{fig:RGTA} offered an illustration of the effect of a set of parameter pairs $(l,u)$ on the $GT_{l,u}(d_{i,j})$ function, the distance matrices as well as data-driven interpretations of the lower bound $l$ and the upper bound $u$ for this stepwise function.

\textbf{Adaptive Geo-Topological Independence Criterion (AGTIC).} 
Before we define our statistics, we always assume the following regularity conditions: (1) $(X, Y)$ have finite second moments, (2) neither random variable is a constant and (3) $(X, Y)$ are continuous random variables, which are also required by dCor to establish convergence and consistency. Since we are using the population definition of the distance correlation, the nonconstant condition ensures a more stable behavior and avoids the trivial case. Given a Geo-Topological Transform $f(\cdot) := GT(\cdot;l,u)$, the population expression for the GT-transformed distance covariance can be defined as:

\vspace{-1.5em}
\begin{align}
\label{eq:dCovf}
\begin{split}
\mathcal{V}^{2*}(X,Y;f) = &\EX[f(\norm{X-X'})f(\norm{Y-Y'})] \\
&+\EX[f(\norm{X-X'})]\EX[f(\norm{Y-Y'})]  \\
&- 2\EX[f(\norm{X-X'})f(\norm{Y-Y''})]
\end{split}
\end{align}
\vspace{-0.5em}

where the same set of parameters $l$ and $u$ applies to the monotonic transforms on all the distances. The GT-transformed distance correlation is then:

\vspace{-1.5em}
\begin{equation}
\label{eq:dCovf}
dCor^*(X,Y;f) = \frac{\mathcal{V}^{2*}(X,Y;f)}{\sqrt{\mathcal{V}^{2*}(X,X;f)\mathcal{V}^{2*}(Y,Y;f)}}
\end{equation}
\vspace{-0.5em}

We can naturally define AGTIC to be the maximum GT-transformed distance correlation within the parameter domain $\mathcal{S} := \{l\in[0,1), u\in(0,1], l<u\}$:

\vspace{-1.5em}
\begin{equation}
\label{eq:AGTIC}
AGTIC(X,Y) = \max_{(l,u)\in\mathcal{S}} dCor^*(X,Y;GT(\cdot;l,u))
\end{equation}
\vspace{-1.5em}

\textbf{Test description.} 
The statistical test of independence can be performed by locating the test statistic in its distribution under $\mathcal{H}_0$ using a permutation procedure \cite{gretton2008kernel}. As the preliminary, given the i.i.d. sample $\mathcal{Z}=(X,Y)=\{(x_1,y_1),\cdots,(x_m,y_m)\}$ defined earlier, the statistical test $\mathcal{T}(Z):=(\mathcal{X}\times\mathcal{Y})^m\rightarrow \{0,1\}$ distinguishes between the null hypothesis $\mathcal{H}_0: \mathbbm{P}_{xy}=\mathbbm{P}_x\mathbbm{P}_y$ and the alternative hypothesis $\mathcal{H}_1: \mathbbm{P}_{xy}\not=\mathbbm{P}_x\mathbbm{P}_y$. This is achieved by comparing the test statistic, in our case $AGTIC(Z)$, with a particular threshold: if the threshold is exceeded, then the test rejects the null hypothesis. The permutation test involves the following steps. Based on a finite sample, incorrect answers can yield two kinds of errors: the Type I error is the probability of rejecting $\mathcal{H}_0$ when $x$ and $y$ are in fact independent, and the Type II error is the probability of accepting $\mathbbm{P}_{xy}\not=\mathbbm{P}_x\mathbbm{P}_y$ when in fact the underlying variables are dependent. To obtain an estimate of the Type I and Type II error, we need to create a null distribution of $Z$, for instance, by shuffling the labels of $X$ or $Y$ such that their one-to-one correspondence are now disconnected, and thus, independent. By computing the statistics multiple times on the null distribution the original dataset, we obtain the null distribution for our test statistics $AGTIC(\text{null}(Z))$, where we can align our test statistics computed for the real data $AGTIC(Z)$ to get Type I error and Type II error. For instance, if our computed $AGTIC(Z)$ is larger than 97\% of the $AGTIC(\text{null}(Z))$, then the Type I error. In practice, we specify a cutoff $\alpha$ for the false positive rates to be the upper bound on the Type I error. We can further define the empirical statistical power as the fraction of true datasets yielding a statistic value greater than 95\% of the values yielded by the corresponding null datasets, with a theoretical guarantee that the false positive rate is below $\alpha =$ 5\%.

\begin{figure}[b]
\vspace{-2em}
\centering
    \includegraphics[width=\linewidth]{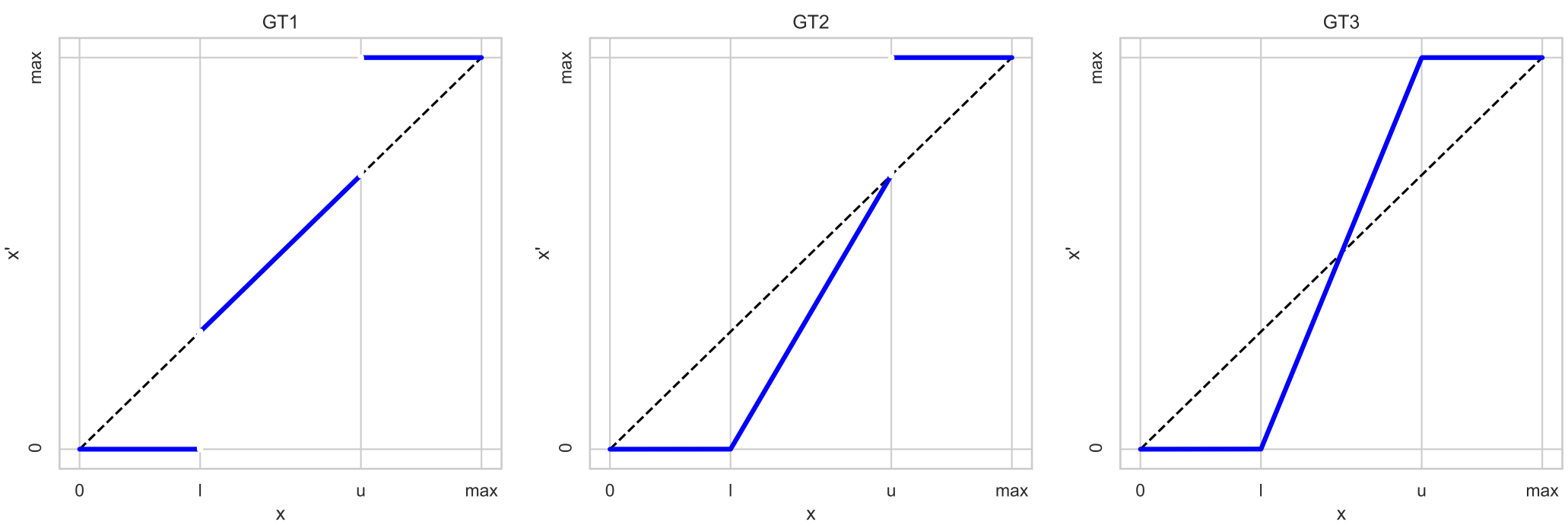}
    \vspace{-1.5em}
\par\caption{\textbf{Example GT transforms: } t1, t2, t3.}\label{fig:GT}
\end{figure}

\textbf{Algorithmic variants.}
There are several algorithmic decisions one can make to compute AGTIC. Here we briefly describe the four types of geo-topological transforms and three types of subroutines used in the empirical evaluation. Instead of manually setting the parameters for the monotonic transformation, we adaptively selects threshold parameters, using the maximum of the statistics over the parameter domain to compute the test statistic. As primed earlier, the nonlinear monotonic transforms can be specified in different formats. One working hypothesis is that only local distances contribute to the mutual information, such that larger distances aren't as relevant to the global topology. In this setting, only small distances are counted as neighbors to form an edge within the topological graph. \textit{AGTIC - t0} sets the lower bound $l$ to zero, and only search for the optimal upper bound $u$ (here the ``optimal'' means yielding the maximum statistics). Fig. \ref{fig:GT} illustrated three other types of monotonic transform that we considered, denoted \textit{AGTIC - t1}, \textit{AGTIC - t2}, \textit{AGTIC - t3}. In this setting (also the default one), only pairs with intermediate distance are counted as neighbors. The logic behind this setup is that the dependence between very proximal data points can be more likely attribute to noise, therefore by ``discrediting'' these edges, we have a more stable topology. Other than the transforms, we also compared three variants of our test statistics, denoted \textit{AGTIC - s1}, \textit{AGTIC - s2}, \textit{AGTIC - s3} (``s'' stands for subroutine):  

\vspace{-1.5em}
\begin{equation}
\label{eq:AGTICv}
\begin{split}
\mathcal{V}^{2*}_{s1} (X,Y) &= \max_{(l,u)\in\mathcal{S}} dCor^*(X,Y;GT(\cdot;l,u)) \\
\mathcal{V}^{2*}_{s2} (X,Y) &= \max_{(l,u)\in\mathcal{S}} \frac{dCor^*(X,Y;GT(\cdot;l,u))}{dCor^*(\text{null}(X),\text{null}(Y);GT(\cdot;l,u))} \\
\mathcal{V}^{2*}_{s3} (X,Y) &= \frac{\max_{(l,u)\in\mathcal{S}} dCor^*(X,Y;GT(\cdot;l,u))}{\sqrt{\text{Var}(\forall_{(l,u)\in\mathcal{S}} dCor^*(X,Y;GT(\cdot;l,u)))}}\\
\end{split}
\end{equation}
\vspace{-0.5em}

where \textit{AGTIC - s2} selects the maximum of the ratio of the transformed dCor for the dataset over the transformed dCor for the null dataset, and \textit{AGTIC - s3} includes an additional noise normalization procedure (dividing the maxmimum test statistics by the standard deviation of all the test statistics in the parameter domain). Last but not least, we also considered the case where the upper and lower thresholds were not cut-offs of a specific distance value, but the cut-off of a specific fraction of the ranked data. In another word, we set the $l$ and $u$ as percentiles instead of scales. We denote this variant \textit{pAGTIC} (where ``p'' stands for ``percentile'').

\textbf{Other weighting schemes.} 
It is worth exploring whether replacing the piece-wise linear transform with an alternative (e.g. differentiable) monotonic transform brings further improvements. However, we might expect such functions to behave similarly and prefer the simplicity of piece-wise linear functions and the fact that they make subsetting distances straightforward. 

\textbf{Hypothesis. }
We hypothesize that the monotonic transforms of distance matrices associated with the two variables offers an additional attention mechanism to the traditional distance correlation, such that data points with small distances are treated as identical (one collapsed node within a topological graph) and data points with large distances are considered disconnected (no matter how distant they are from each other). Rather than simply thresholding the distance matrix, which replaces a geometrical summary with a topological summary, we explore transforms that can suppress variations among small distances (which tend to be dominated by noise) and among large distances (which may not reflect mutual information between the two variables), while preserving geometrical information (which may boost sensitivity to simple relationships). We refer to these transforms as geo-topological transforms, because they combine aspects of geometry and topology. Depending on the choice of the lower and upper bounds, these transforms can threshold (lower bound $l$ = upper bound $u$) at arbitrary levels, adapting to the granularity of the dependency present in a data set. They can also, optionally, preserve geometrical information (lower bound $l$ < upper bound $u$).

\section{Properties of AGTIC}
\label{sec:properties}

\textbf{Independence indicator.} dCor is an independence indicator such that: (1) $\mathcal{R}(X,Y)$ defined for $X, Y$ in arbitrary dimensions and (2) $\mathcal{R}(X,Y)=0 \Leftrightarrow X$ and $Y$ are independent \cite{szekely2007measuring,szekely2009brownian}. As in Appendix \ref{sec:indeptheory}, we have:

\begin{theorem}
$\mathcal{V}^{2*}(X,Y;GT(\cdot;l,u)=0$ when $X$ and $Y$ are independent, given any $(l,u) \in \mathcal{S}$. Moreover, at $(l,u)=(0,1)$, $\mathcal{V}^{2*}(X,Y;GT(\cdot;l,u))=\mathcal{V}^{2}(X,Y)$. They also hold for the distance correlation by replacing the distance covariance $\mathcal{V}^{2*}(X,Y;GT(\cdot;l,u))$ with $dCor^*(X,Y;GT(\cdot;l,u))$.
\end{theorem}

\textbf{Computational complexity.}
Here, we consider the most computationally demanding 
of the family, \textit{AGTIC - s3}, which consist of a combinatorial threshold space and a noise normalization. In the typical setup (very large sample size $m$ and small number of thresholds $k$), the computational complexity is dominated by the threshold searching with two variables. Hence, we achieve a cost in terms of the sample size of $\mathcal{O}(m^2(k(k-1)/2)^2) \approx \mathcal{O}(m^2k^4)$. In the special case of the distance covariance with univariate real-valued variables, \cite{huo2016fast} achieve an $\mathcal{O}(m\log m)$ cost for dCor computation, thus potentially reducing complexity for AGTIC to $\mathcal{O}(m\log m(k(k-1)/2)^2) \approx \mathcal{O}(mk^4\log m)$.

\textbf{Cheaper options via sampling with convergence.}
Despite the fact that the number of thresholds $k$ is generally small (around 5 to 10), $\mathcal{O}(k^4)$ is still considerably larger than the vanilla dCor. \cite{szekely2007measuring,szekely2014partial} showed that sample dCor can be easily computed to converge to the population dCor via properly centering the Euclidean matrices. Similarly as \cite{szekely2014partial,shen2019distance}, sample AGTIC can also be computed via Euclidean distance matrices after the monotonic transform and the sample version converges to the respective AGTIC up to a difference of $\mathcal{O}(\frac{1}{n})$ where n is the number of sample of threshold sets.

\textbf{Relationship to RKHS-based statistics.}
Here we state that monotonically transformed distance metrics (such as ${\mathcal{V}}^{2*}(X,Y)$) can be defined as a \textit{distance-induced kernel} \cite{sejdinovic2013equivalence}, a special case in RKHS-based independence statistics. Proof follows: For fixed $n \geq 4$, distance correlation is defined in a Hilbert Space generated by Euclidean distance matrices of arbitrary sets (samples) of $n$ points in a Euclidean space $\mathbb{R}^p$, $p\geq 1$ \cite{szekely2014partial}, such that for each pair of elements $C = (C_{i,j}),D = (D_{i,j})$ in the linear span of $\mathcal{H}_n=\{\tilde{A}:A\in\mathcal{S}_n\}$ where $\mathcal{S}_n$ is the linear span of all $n\times n$ distance matrices of samples $\{x_1,\cdots,x_n\}$, empirical inner product is defined as:

\vspace{-1.5em}
\begin{equation}
\label{eq:innerprod}
\begin{split}
\langle C, D\rangle & = \frac{1}{n(n-3)}\sum\limits_{i\not=j}C_{ij}D_{ij} \\
\end{split}
\end{equation}
\vspace{-1em}

In our case, ${\mathcal{V}^*}^2(X,Y)$ is not necessarily still defined in a Hilbert Space (as ${\mathcal{V}}^2(X,Y)$ in \cite{szekely2007measuring}), because $f_n(x)$ is a monotone nonlinear operator (as defined by \cite{minty1962monotone}):

\vspace{-2em}
\begin{equation}
\label{eq:monotone}
\begin{split}
\langle y-x, f(y)-f(x)\rangle \geq 0, \text{\space} & \forall x,y \in \mathcal{H}
\end{split}
\end{equation}
\vspace{-1.5em}

\begin{theorem}
Give $f_n(\mathcal{X})$ is a monotone nonlinear operator on a Hilbert Space $\mathcal{X}$, then the kernel of $f(\mathcal{H})$ is still continuously defined to be valid within Hilbert Space $\mathcal{X}$.
\end{theorem}

\begin{proof}
\cite{minty1962monotone} further defined the (not necessarily linear) \textit{monotone operator} $f$ as \textit{maximal} if it cannot be extended to a properly larger domain while preserving its monotoneity, which in our case, is the maximum value cap in the geo-topological transformation. In \cite{minty1962monotone}, Theorem 4 Corollary states: If $f: \mathcal{D}\to \mathcal{X}$ is a continuous monotone operator, then $(I+f)^{-1}$ exists, is continuous on its domain, and is monotone; if in addition, $f$ is continuous and maximal, and has open domain  $\mathcal{D}$ (in particular, if $f$ is continuous and \textit{everywhere-defined}), then $(I+f)^{-1}$ is \textit{everywhere-defined}. This shows that despite the fact that the distance correlation after our proposed family of geo-topological transformation is no longer an inner product space, it is sufficient to show that a mapping $\phi$ exists to transform back to the original Hilbert Space such that the kernel operations are everywhere defined valid. As an extension, this theorem applies to other possible monotone operations such as generalized logistic function and sigmoid functions. As we showed in Equation \ref{eq:supConvergence2} that this transformation is \textit{complete} in the space $L^2[a,b]$, ${\mathcal{V}^*}^2_{max}(X,Y)$ can still maintain the kernel properties for an inner product space which is \textit{complete} (as a metric space), and in another word, a \textit{Hilbert Space}. 
\end{proof}

It was suggested that distance-based and RKHS-based statistics are fundamentally equivalent for testing dependence \cite{sejdinovic2013equivalence}. Here we follow their logic to explore any relationship between RKHS and AGTIC. According to \cite{berlinet2011reproducing}, for every symmetric positive definite function (i.e. \textit{kernel}) $k:\mathcal{Z}\times \mathcal{Z}\to \mathbb{R}$, exists an associated RKHS $\mathcal{H}_k$ of real-valued functions on $\mathcal{Z}$ with reproducing kernel $k$. Given $v\in \mathcal{M}(\mathcal{Z})$, the \textit{kernel embedding} of $v$ into the RKHS $\mathcal{H}_k$ is defined as $\mu_k(v) \in \mathcal{H}_k$ such that $\int f(z)dv(z)= \langle f, \mu_k(v)\rangle_{\mathcal{H}_k} $ held for all $f \in \mathcal{H}_k$ \cite{sejdinovic2013equivalence}. 

\begin{lemma}
In order to define a \textit{distance-induced kernel} $k(z,z')=\frac{1}{2}[\rho(z,z_0)+\rho(z',z_0)+\rho(z,z')]$ for $z_0 \in \mathcal{Z}$, $\rho$ should be a semi-metric of negative type \cite{sejdinovic2013equivalence}. 
\end{lemma}

\cite{lyons2013distance} showed that for distance-based independence testing, it is necessary and sufficient that the metric space be of \textit{strong negative type}, which only holds for \textit{separable} Hilbert Spaces. Here we investigate the separability:

\begin{theorem}
If the geo-topological transform $f_n(\mathcal{X})$ is a continuous monotone operator on a \textit{separable} Hilbert Space $\mathcal{X}$ (distance metric), then it defines a \textit{separable} space.
\end{theorem}

\begin{proof}
A topological space is called \textit{separable} if it contains a \textit{countable}, \textit{dense} subset. In our case, given the countable set $\mathcal{X}$ (original distance) and a function $f:\mathcal{X} \to \mathcal{X}'$ which is \textit{surjective} on $\mathcal{X}'$ (the Hilbert Space we just defined), then $\mathcal{X}'$ is \textit{finite} or \textit{countable}. Then we need to prove that the dense subset $\mathcal{S}$, (in our case, $GT(·;l,u)$), projected from the original distance metric $\mathcal{X}$ (which is a dense subset and \textit{Hausdorff Space}) through the geo-topological transformation is still a \textit{dense} subset of the topological space: since $\mathcal{S} \subset \mathcal{X} \Rightarrow f(\mathcal{S}) \subset f(\mathcal{X})$, and from the dense property we have $f(\mathcal{X}) \subset f(\bar{\mathcal{S}})$, and since $f$ is continuous, $f(\mathcal{X})\subset f(\bar{\mathcal{S}})\subset \bar{f(\mathcal{S})}$, then we proved that $f(\mathcal{S})$ is \textit{dense} in $f(\mathcal{X})$. 
\end{proof}

\begin{corollary}
Since ${\mathcal{V}}^{2*}(X,Y)$ is defined within a \textit{separable} Hilbert Space, it is a \textit{semi-metric of negative type}, and can therefore define a \textit{distance-induced kernel}.
\end{corollary}

We point out that this property suggests that the proposed family of AGTIC can be extended to generalized kernel-based measures, accounting for distances between embeddings of distributions such as maximum mean discrepancies (MMD). 
\textit{Theoretical contributions:} 
In summary, we show that like dCor, AGTIC is a proper independence criterion that returns zero when the two input random variables are independent (Theorem 4.1), evaluates its computational complexity, show that it corresponds to an inner-product on some RKHS (Theorem 4.2) that is separable (Theorem 4.4) and therefore is a distance-induced kernel (Corollary 4.4.1). 
Unfortunately, we do not presently have theoretical results suggesting a general superiority to dCor. We discuss that our method provides no free lunch. Our empirical study, however, shows that the AGTIC, thanks to its adaptivity, performs robustly across different types of statistical dependency.

\section{Empirical Evaluation}
\label{sec:results}

\begin{table*}[tb]
	\centering
	\caption{\textbf{Summary} of power of different statistical tests (rows) for detecting relationships (columns)}
	\resizebox{1\linewidth}{!}{
		\begin{tabular}{ l | l | l | l | l | l | l | l }
EXP & TEST	&  linear & parabolic & sinusoidal & circular & checkerboard & average \\ 
			 \thickhline
\multirow{8}{*}{noise levels} &	AGTIC	&	0.674	$\pm$	0.351	&	\textbf{0.602}	$\pm$	\textbf{0.373}	&	\textbf{0.712}	$\pm$	\textbf{0.384}	&	0.688	$\pm$	0.414	&	\textbf{0.618}	$\pm$	\textbf{0.410}	&	\textbf{0.628}	$\pm$	\textbf{0.396}	\\
& MI	&	0.502	$\pm$	0.404	&	0.534	$\pm$	0.354	&	0.586	$\pm$	0.381	&	0.702	$\pm$	0.391	&	0.602	$\pm$	0.423	&	0.565	$\pm$	0.386	\\
& dCor	&	0.676	$\pm$	0.353	&	0.550	$\pm$	0.425	&	0.452	$\pm$	0.424	&	0.446	$\pm$	0.448	&	0.210	$\pm$	0.140	&	0.467	$\pm$	0.392	\\
& Hoeffding's D	&	0.650	$\pm$	0.356	&	0.460	$\pm$	0.414	&	0.460	$\pm$	0.431	&	0.498	$\pm$	0.475	&	0.200	$\pm$	0.112	&	0.454	$\pm$	0.393	\\
& HSIC	&	0.504	$\pm$	0.416	&	0.556	$\pm$	0.363	&	0.324	$\pm$	0.369	&	0.670	$\pm$	0.439	&	0.196	$\pm$	0.082	&	0.450	$\pm$	0.383	\\
& MIC	&	0.344	$\pm$	0.335	&	0.378	$\pm$	0.299	&	0.586	$\pm$	0.288	&	0.438	$\pm$	0.366	&	0.310	$\pm$	0.193	&	0.411	$\pm$	0.305	\\
& rdmCor	&	0.426	$\pm$	0.406	&	0.534	$\pm$	0.420	&	0.028	$\pm$	0.023	&	\textbf{0.728}	$\pm$	\textbf{0.408}	&	0.036	$\pm$	0.042	&	0.350	$\pm$	0.415	\\
& R$^2$	&	\textbf{0.710}	$\pm$	\textbf{0.325}	&	0.136	$\pm$	0.104	&	0.054	$\pm$	0.053	&	0.028	$\pm$	0.036	&	0.072	$\pm$	0.043	&	0.200	$\pm$	0.300	\\ 
            \hline
        \multirow{8}{*}{sample sizes} &	    AGTIC & \textbf{1.000} $\pm$ \textbf{0.000} & \textbf{1.000} $\pm$ \textbf{0.000} & \textbf{1.000} $\pm$ \textbf{0.000} & \textbf{1.000} $\pm$ \textbf{0.000} & 0.950 $\pm$ 0.218 & \textbf{0.990} $\pm$ \textbf{0.022} \\
& MI & 0.995 $\pm$ 0.011 & 0.985 $\pm$ 0.019 & 0.991 $\pm$ 0.015 & 0.995 $\pm$ 0.010 & \textbf{0.983} $\pm$ \textbf{0.018} & 0.989 $\pm$ 0.005 \\
& Hoeffding's D & \textbf{1.000} $\pm$ \textbf{0.000} & \textbf{1.000} $\pm$ \textbf{0.000} & \textbf{1.000} $\pm$ \textbf{0.000} & \textbf{1.000} $\pm$ \textbf{0.000} & 0.550 $\pm$ 0.497 & 0.910 $\pm$ 0.201 \\
& MIC & 0.984 $\pm$ 0.015 & 0.977 $\pm$ 0.022 & 0.956 $\pm$ 0.160 & 0.891 $\pm$ 0.292 & 0.733 $\pm$ 0.422 & 0.908 $\pm$ 0.105 \\
& dCor & \textbf{1.000} $\pm$ \textbf{0.000} & \textbf{1.000} $\pm$ \textbf{0.000} & 0.900 $\pm$ 0.300 & 0.800 $\pm$ 0.400 & 0.600 $\pm$ 0.490 & 0.860 $\pm$ 0.167 \\
& HSIC & \textbf{1.000} $\pm$ \textbf{0.000} & \textbf{1.000} $\pm$ \textbf{0.000} & 0.850 $\pm$ 0.357 & 0.950 $\pm$ 0.218 & 0.500 $\pm$ 0.500 & 0.860 $\pm$ 0.210 \\
& rdmCor & \textbf{1.000} $\pm$ \textbf{0.000} & \textbf{1.000} $\pm$ \textbf{0.000} & 0.300 $\pm$ 0.458 & 0.000 $\pm$ 0.000 & 0.200 $\pm$ 0.400 & 0.500 $\pm$ 0.469 \\
& R$^2$ & \textbf{1.000} $\pm$ \textbf{0.000} & 0.350 $\pm$ 0.477 & 0.000 $\pm$ 0.000 & 0.000 $\pm$ 0.000 & 0.150 $\pm$ 0.357 & 0.300 $\pm$ 0.417 \\                 
		\end{tabular}
	}  
\label{tab:summary}
\end{table*}

\begin{figure}[tb]
\centering
    \includegraphics[width=1\linewidth]{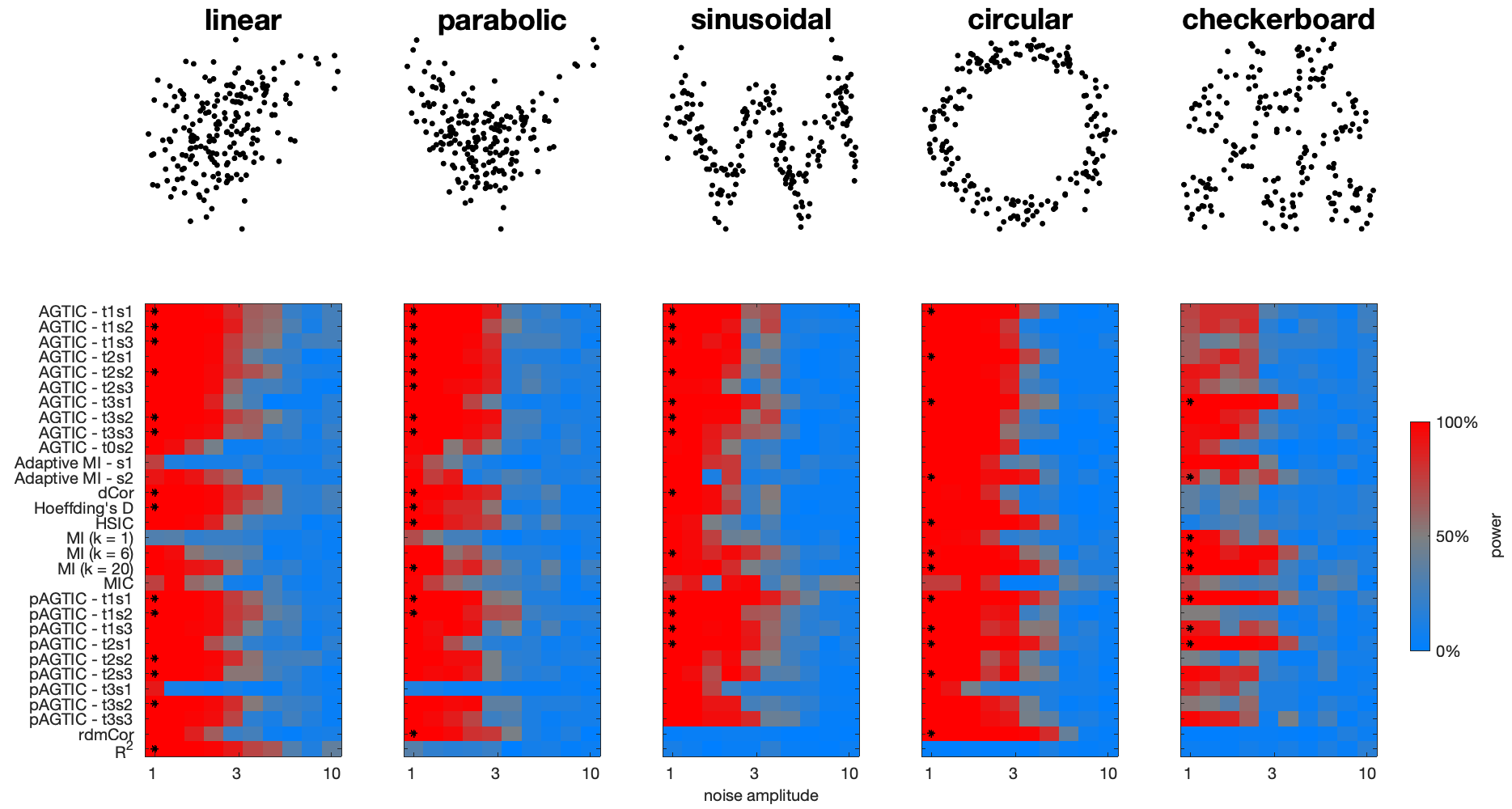}
    \vspace{-1.5em}
\par\caption{\textbf{Statistical power in five association patterns.} Power (color-coded) of different tests (rows) for detecting different forms of dependency (panels) over different noise levels  (horizontal axes). For each pattern, an asterisk indicates that the test retains 50\%-power at a noise level within 25\% of the most sensitive test.}\label{fig:power_noise}
\vspace{-1em}
\end{figure}

\begin{figure*}[tb]
\centering
    \includegraphics[width=0.115\linewidth]{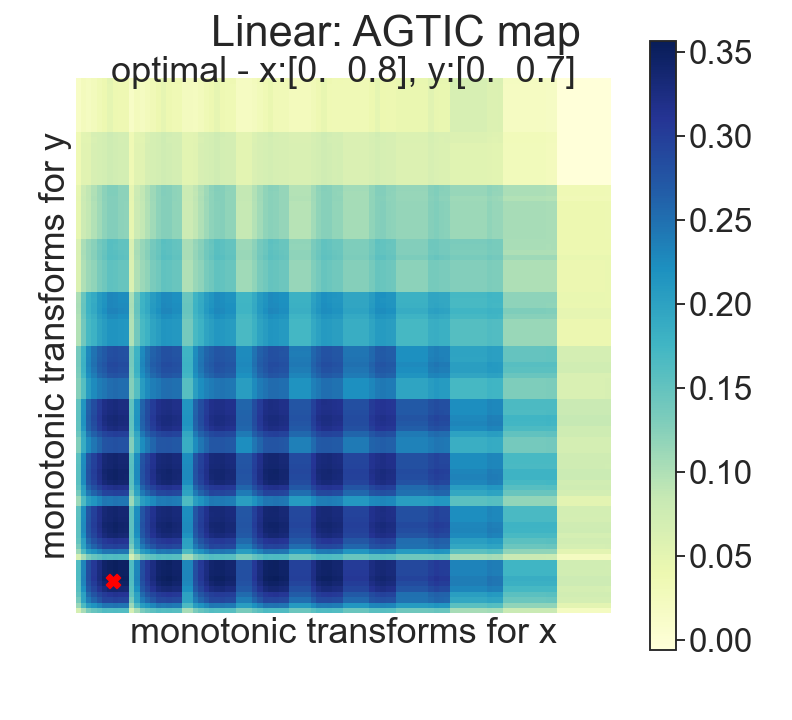}
    \includegraphics[width=0.115\linewidth]{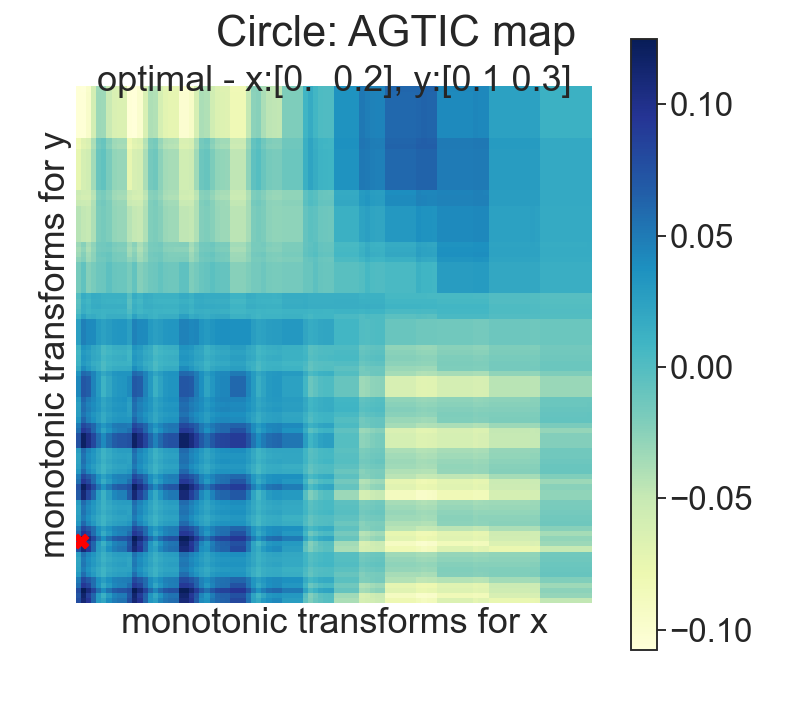}
    \includegraphics[width=0.115\linewidth]{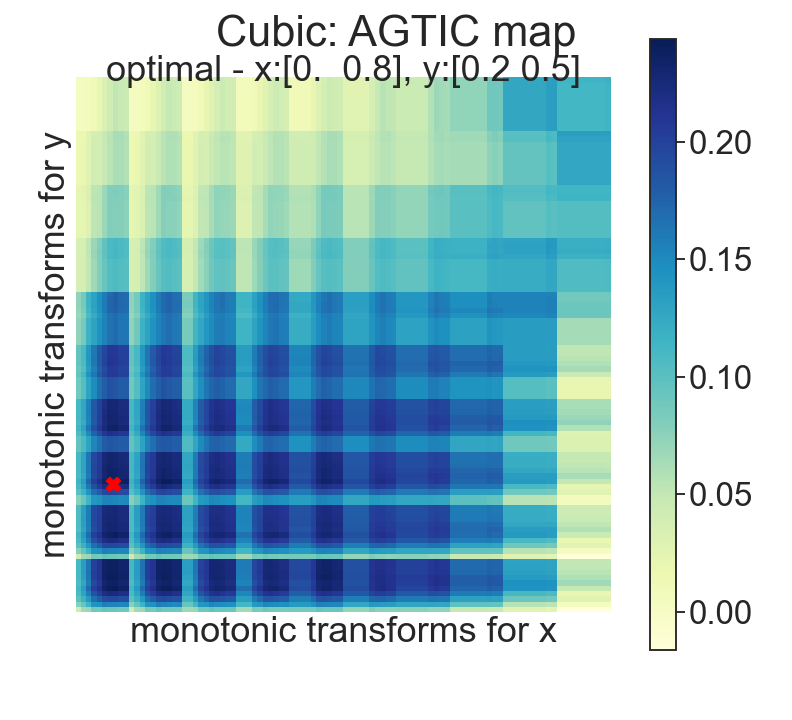}
    \includegraphics[width=0.115\linewidth]{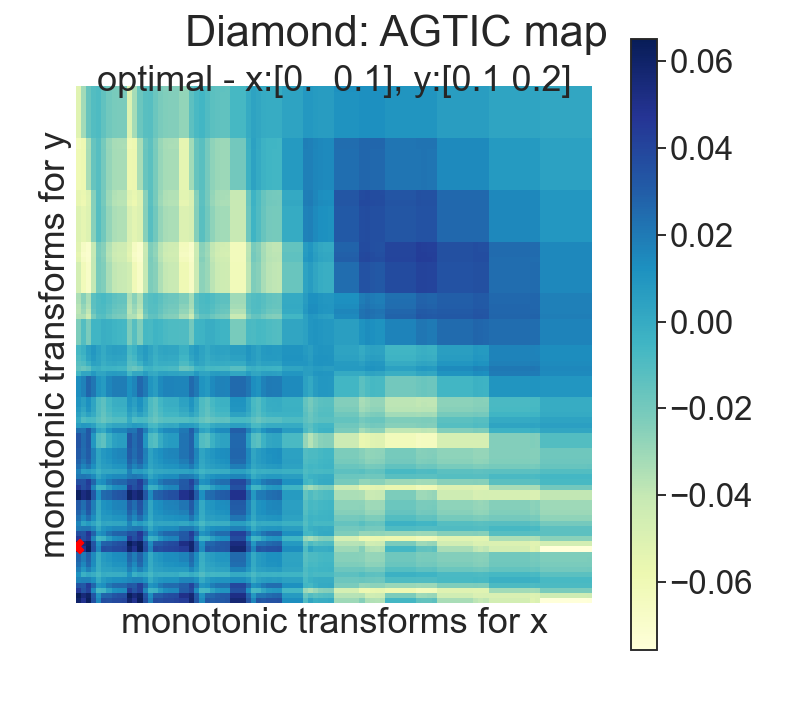}
    \includegraphics[width=0.115\linewidth]{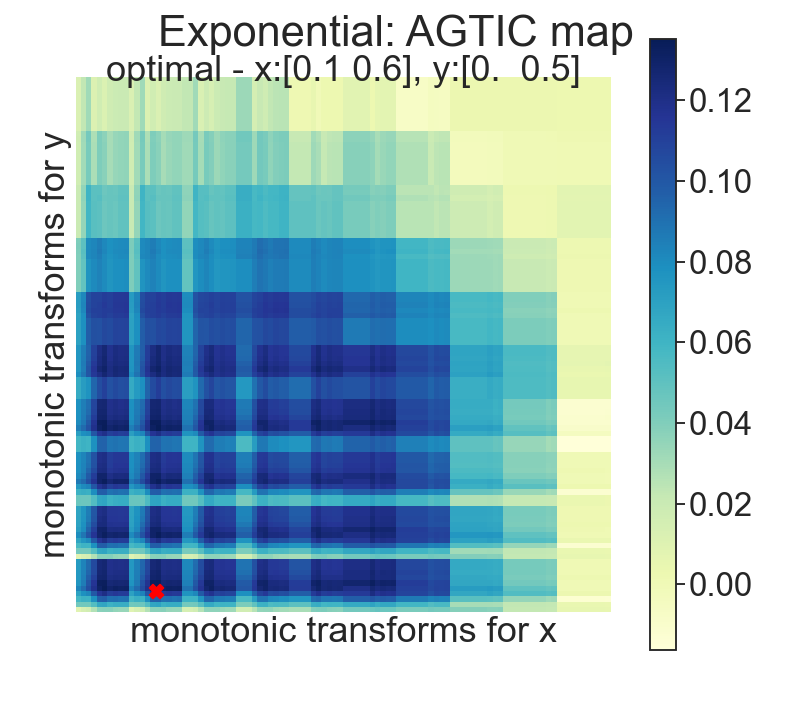}
    \includegraphics[width=0.115\linewidth]{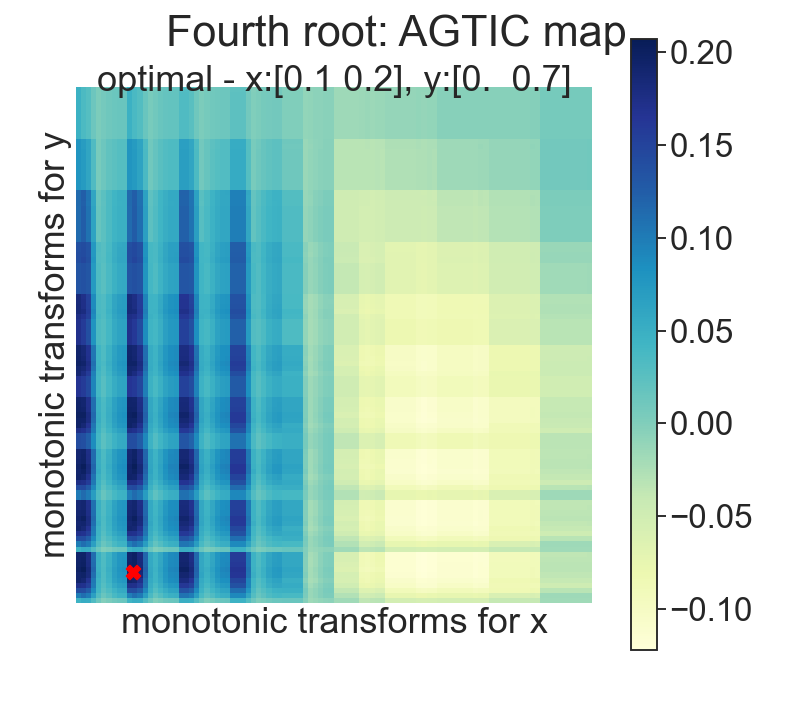}
    \includegraphics[width=0.115\linewidth]{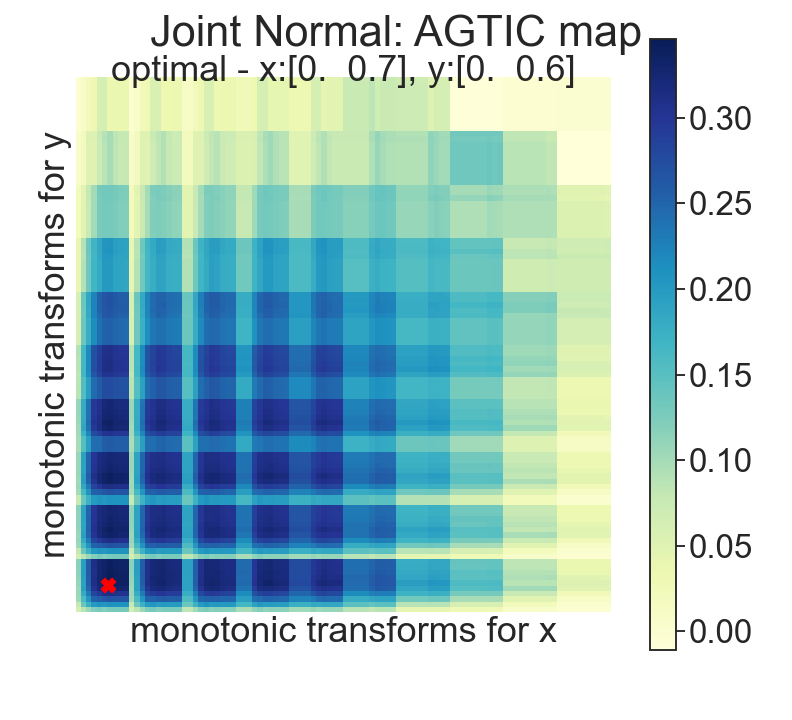}
    \includegraphics[width=0.115\linewidth]{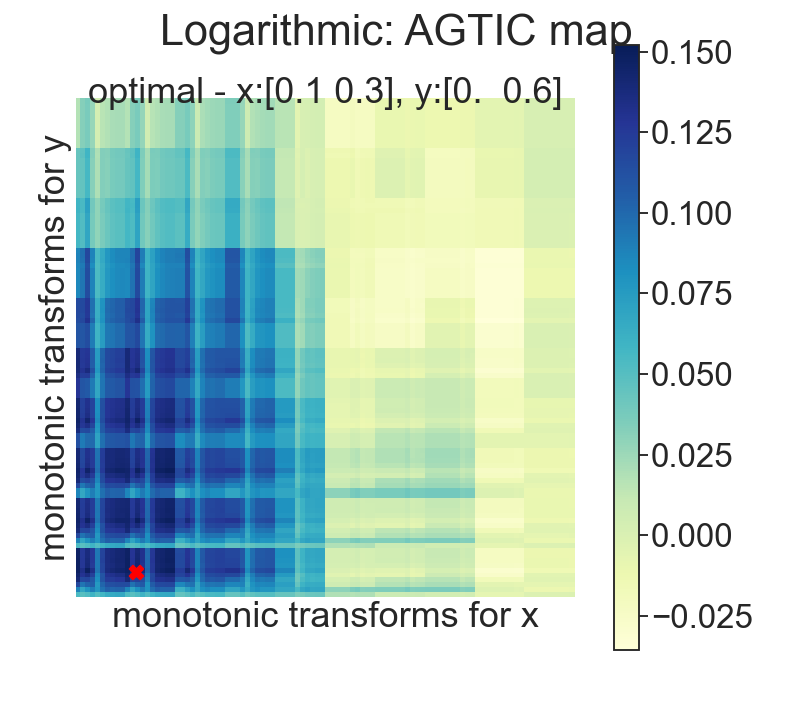}
    \includegraphics[width=0.115\linewidth]{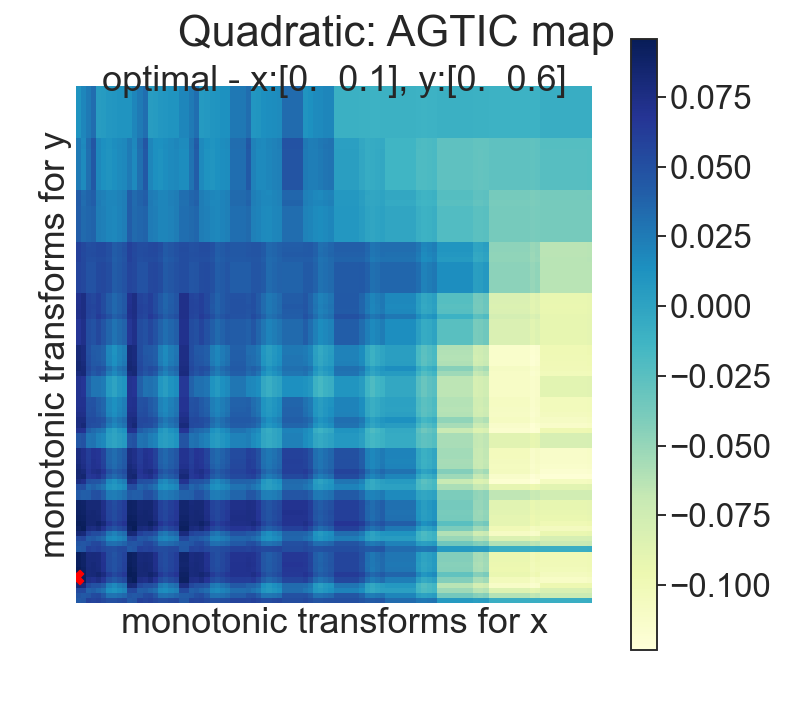}
    \includegraphics[width=0.115\linewidth]{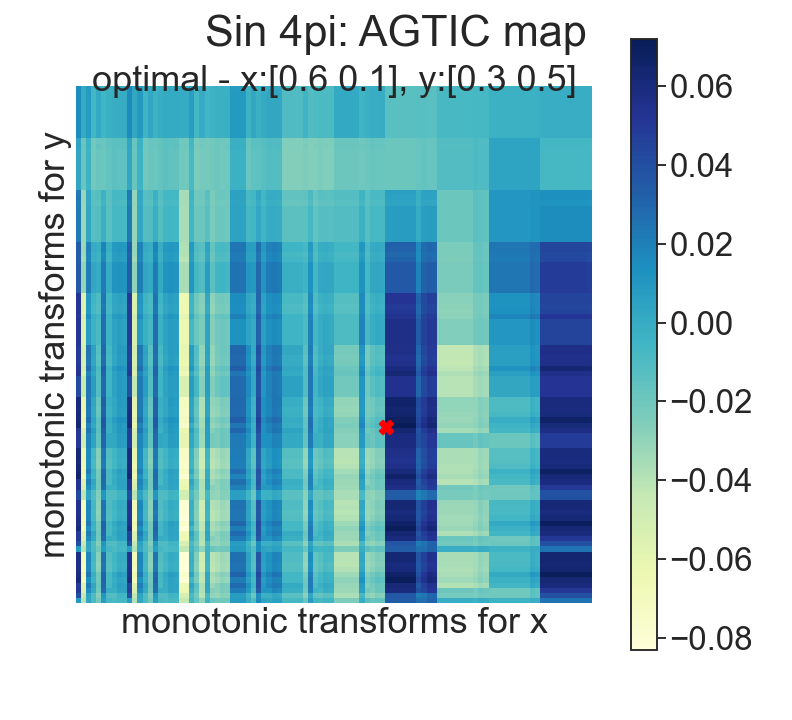}
    \includegraphics[width=0.115\linewidth]{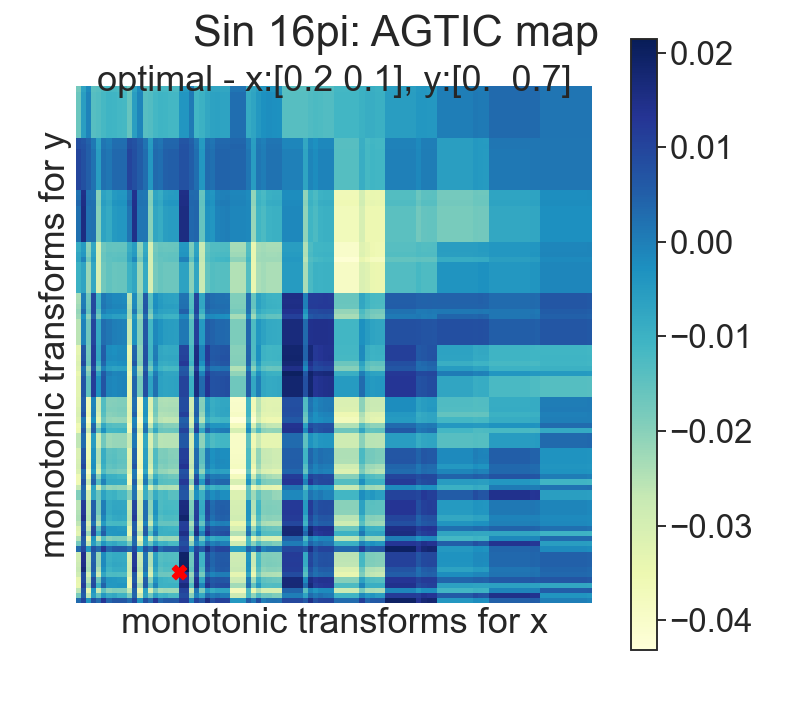}
    \includegraphics[width=0.115\linewidth]{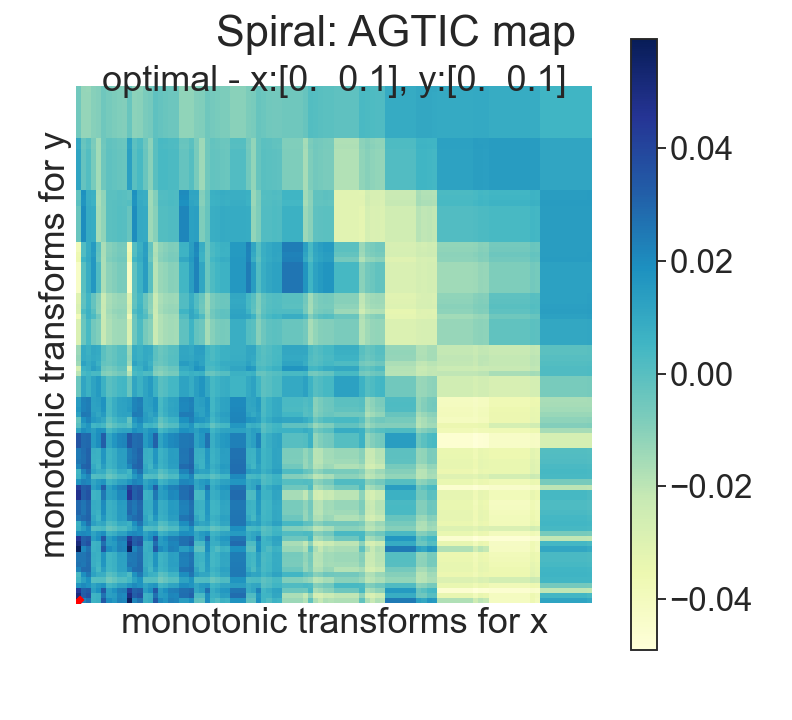}
    \includegraphics[width=0.115\linewidth]{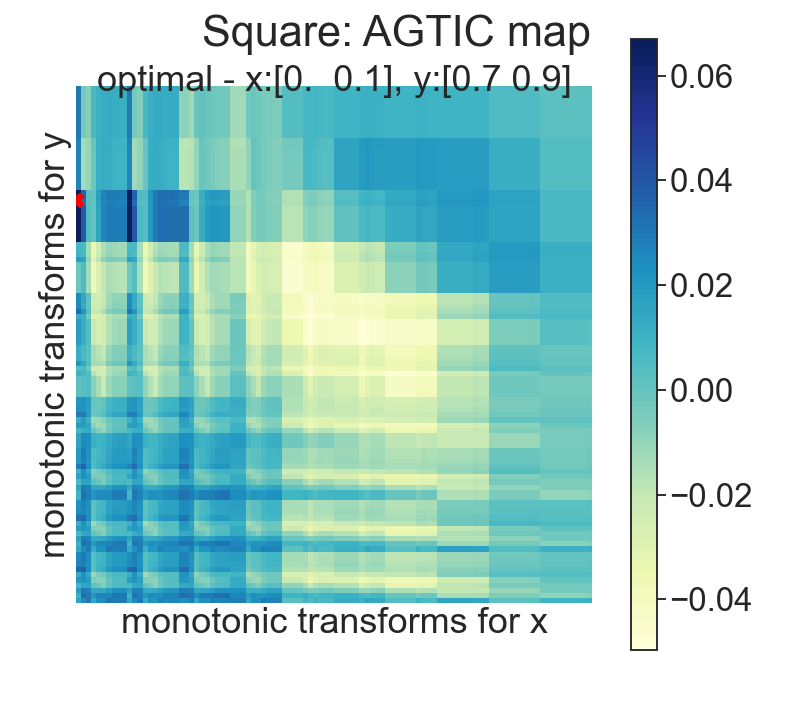}
    \includegraphics[width=0.115\linewidth]{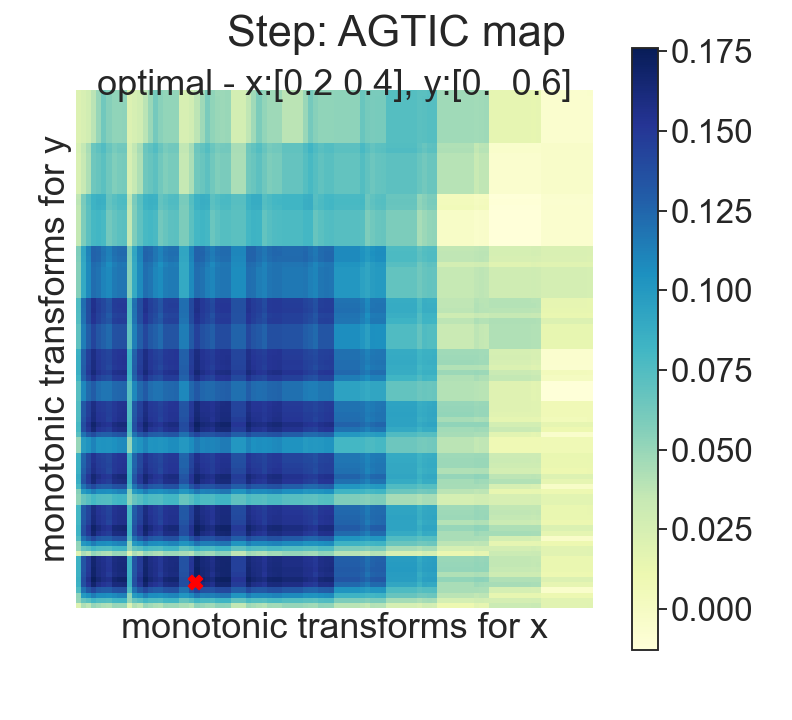}
    \includegraphics[width=0.115\linewidth]{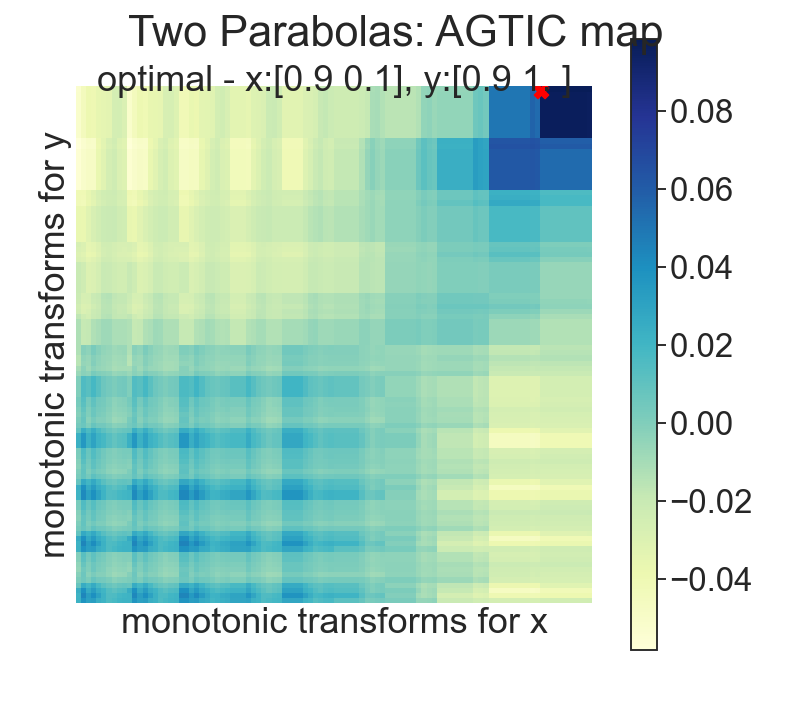}
    \includegraphics[width=0.115\linewidth]{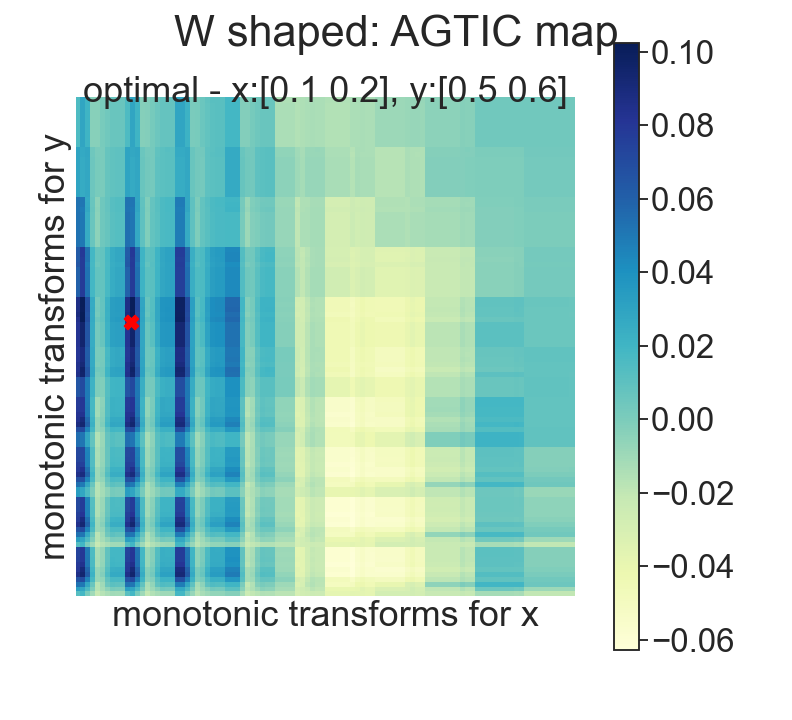}
    \vspace{-1em}
\par\caption{\textbf{The AGTIC maps provide insights to the distinctive geo-topology of each dependency. } The x and y axes of each heatmap denote the sets of threshold pairs $(l,u)$, the color denotes the magnitude of $\mathcal{V}^{2*}(X,Y;GT(\cdot;l,u))$ and the red cross denote the maximum point (i.e. the optimal threshold pairs).}\label{fig:map}
\end{figure*}

\begin{figure}[tbh]
         \vspace{-1.5em}
\centering
    \includegraphics[width=0.75\linewidth]{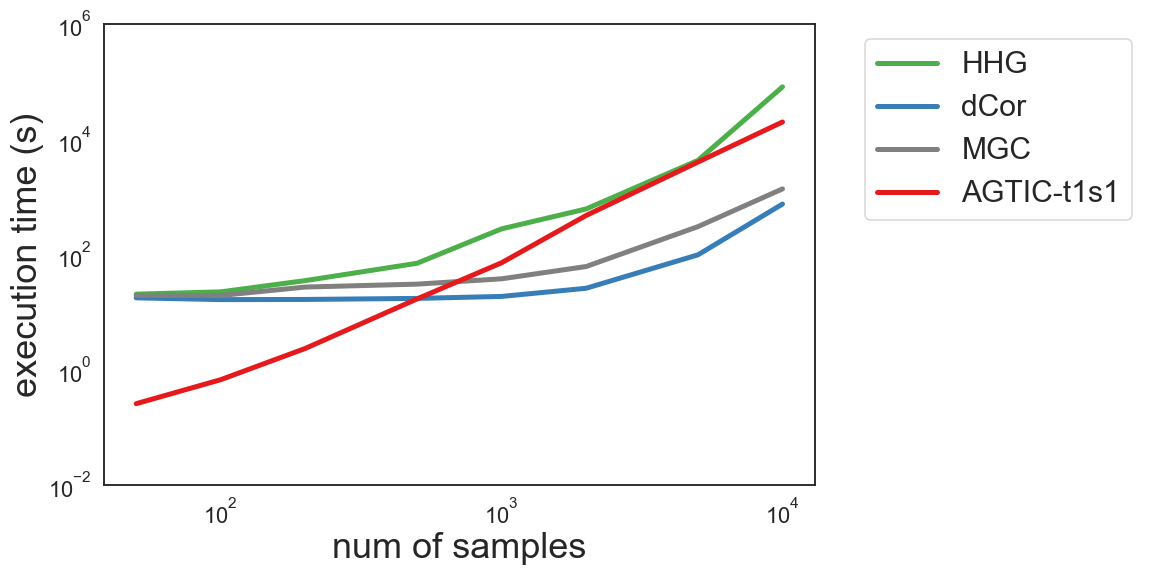}
         \vspace{-0.5em}
\par\caption{\textbf{Run time:} \textbf{(e)} execution wall time of AGTIC, dCor [1,2], MGC [11] and HHG [12] over different numbers sample sizes of 1d linear dependency (with AGTIC in its most time consuming variant).}\label{fig:time} 
         \vspace{-1em}
\end{figure}

\begin{figure*}[tbh]
\centering
    \includegraphics[width=0.17\linewidth]{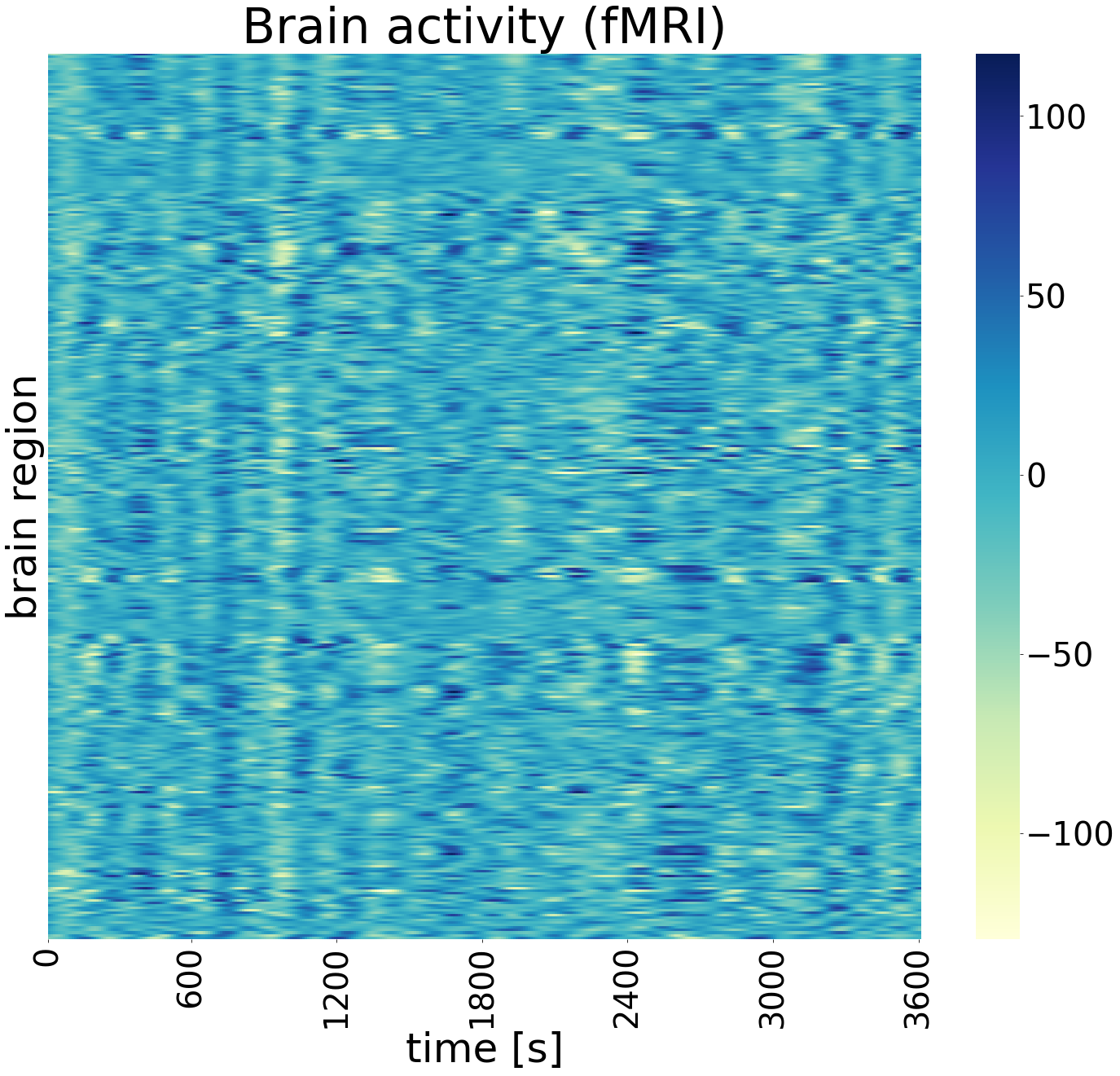}\hfill
    \includegraphics[width=0.19\linewidth]{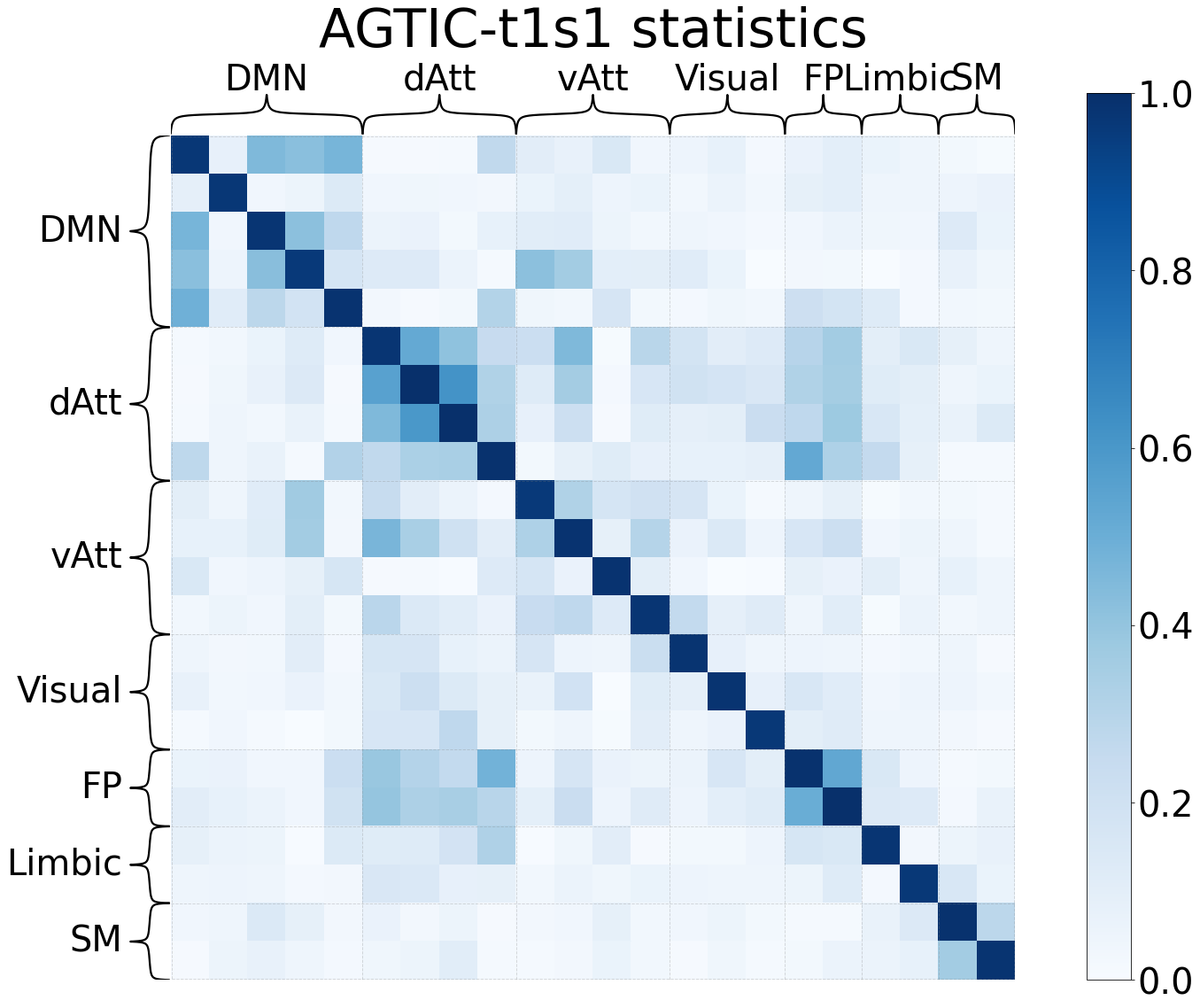}\hfill
    \includegraphics[width=0.19\linewidth]{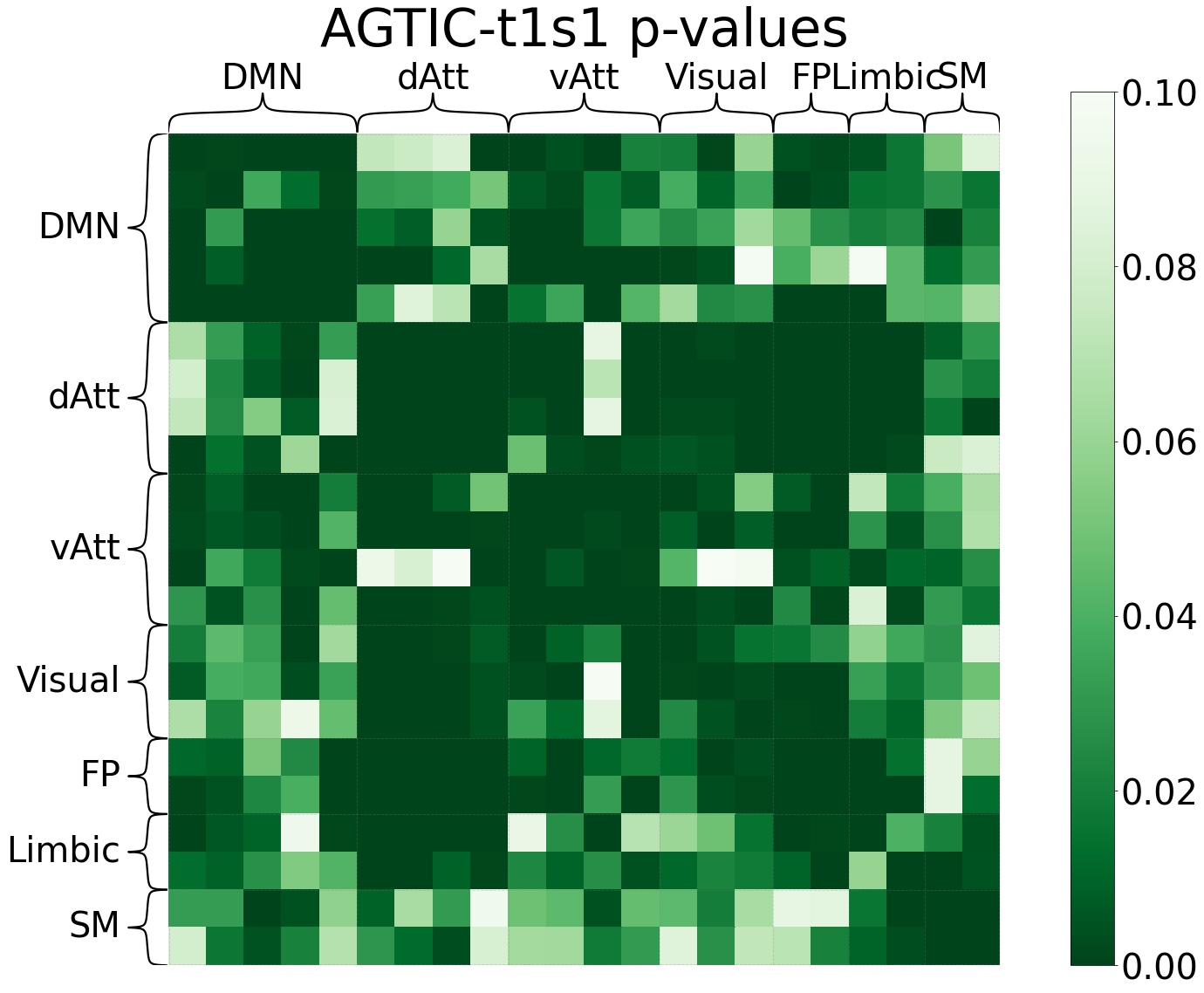}\hfill
    \includegraphics[width=0.19\linewidth]{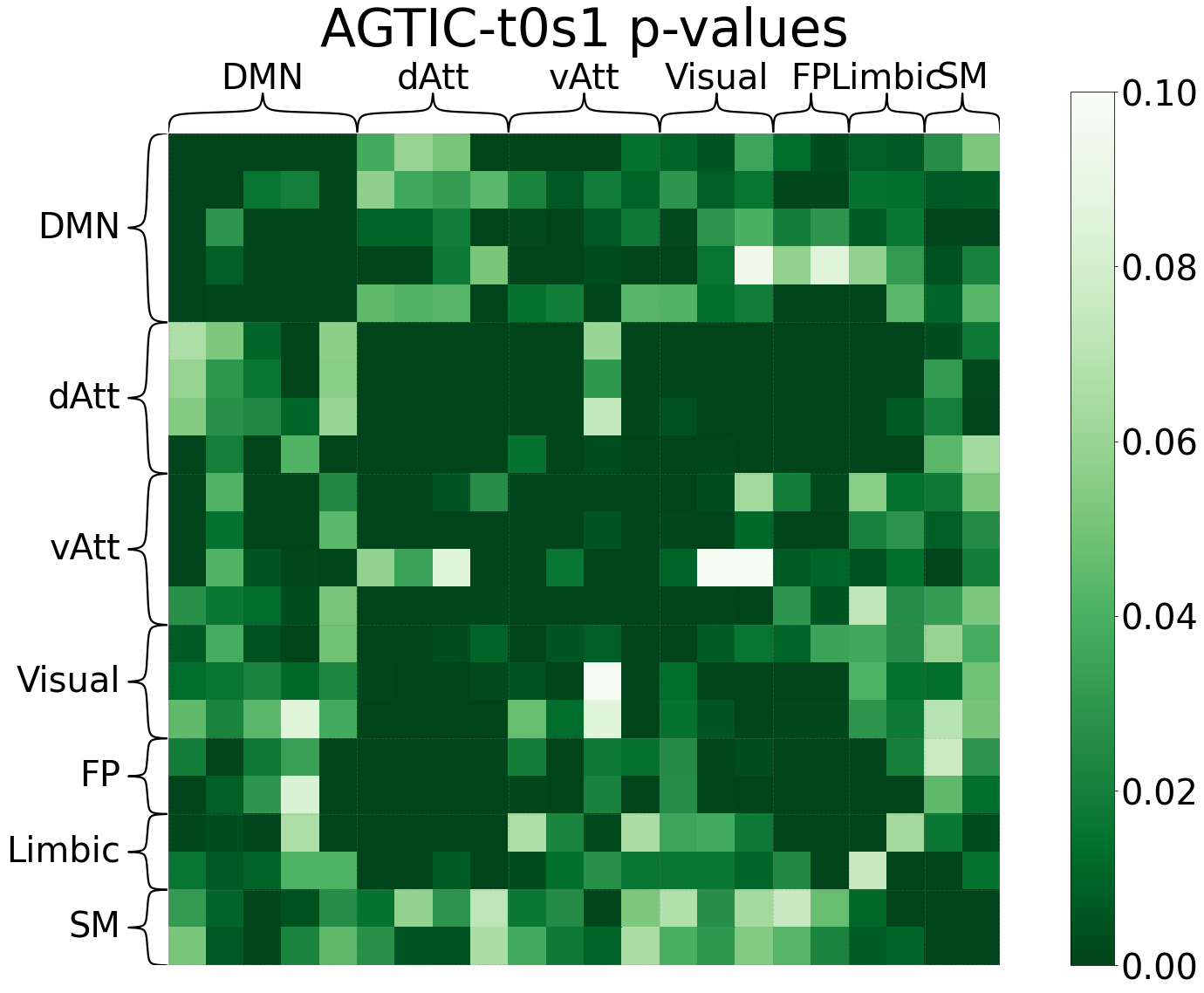}\hfill
    \includegraphics[width=0.19\linewidth]{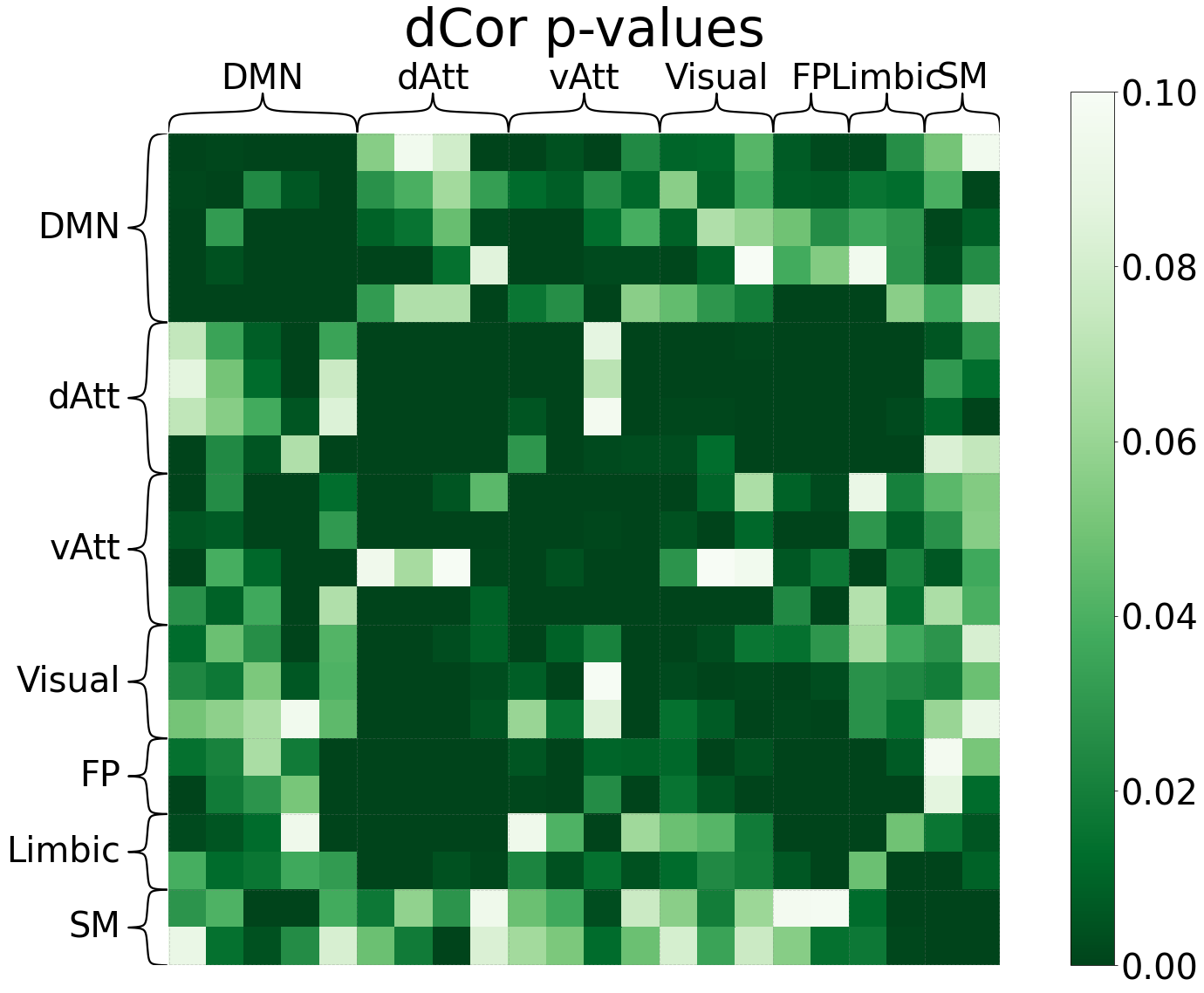}\hfill
         \vspace{-0.5em}
\par\caption{\textbf{Neural data:} \textbf{(b)} HCP data \textbf{(b)} AGTIC among brain regions; \textbf{(c,d)} AGTIC p-values; \textbf{(e)} dCor p-values.}\label{fig:real} 
         \vspace{-1em}
\end{figure*}

\textbf{Evaluation metric and benchmarks.}
In spirit of \textit{no free lunch} in Statistics, \cite{simon2014comment} stressed the importance of statistical power to evaluate the capacity to detect bivariate association. In our context, the \textit{statistical power} of a dependence measure is the fraction of data sets generated from a dependent joint distribution that yield a significant result (with the false-positives rate controlled at 5\%). \cite{simon2014comment} and \cite{kinney2014equitability} compared several independence measures and showed that dCor \cite{szekely2007measuring, szekely2009brownian} and KNN mutual information estimates (MI) \cite{kraskov2004estimating} have substantially more power than MIC \cite{reshef2011detecting,reshef2013equitability}, but adaptive approaches like ADIC were neither proposed nor tested. To understand the behavior of these adaptive dependence measures, we investigated whether their statistical power can compete with dCor, MIC, MI and representational dissimilarity matrix correlation (rdmCor) \cite{kriegeskorte2008representational}. As a fair comparison, other than the nonparametric or rank-based methods, here we adopted our maximum selection subroutines to one of the popular parametric dependence measure, K-nearest neighbour mutual information estimator, such that the hyperparameter $k$ is adaptively defined, denoted \textit{Adaptive MI - s1} and \textit{Adaptive MI - s2} (more details on parameter settings in Appendix \ref{sec:paras}).

\textbf{Multivariate association patterns. }
We considered sixteen common multivariate patterns (Fig. S\ref{fig:data}) and evaluate the statistical power of the statistics with different additive noises, sample sizes, dimensions and combinatorial dependencies. Here we report five distinct patterns: linear, parabolic, sinusoidal, circular and checkerboard, and describe the \textit{free lunch}: which is best, where, and how.



\textbf{Resistance to additive noise.} Fig. \ref{fig:power_noise} shows the assessment of statistical power for the competing nonlinear dependence measures as the variance of a series of zero-mean Gaussian noise amplitude which increases logarithmically over a 10-fold range. The heat maps show power values computed for each statistics.
For each pattern, the asterisks indicate that the statistic that have a noise-at-50\%-power that lies within 25\% of this maximum. Among all competing measures, our proposed AGTIC family ranked best in 4 out of 5 relationships (except linear) and best by average (Table \ref{tab:summary}). 
As expected, R$^2$ was observed to have optimal power on the linear relationship, but it is worth noting that all the AGTIC or pAGTIC algorithms adapt to the linear pattern by choosing the most informative threshold pairs to reach a near optimal performance, while R$^2$ shows negligible power on the other relationships which are mirror symmetric as expected. rdmCor as the correlation coefficient on the pairwise distances of the data, shows optimal power in the circular relationship, but poor performance in all others. The behaviors of dCor and Hoeffding's D are very similar across all relationships, and maintained substantial statistical power on all but the checkerboard relationships. On all but the sinusoidal relationship, MIC with $B = N^{0.6}$ as suggested by \cite{reshef2011detecting} was observed to have relatively low statistical power, consistent with the findings of \cite{simon2014comment} and \cite{kinney2014equitability}. The overall performance of the KNN mutual information estimator using k = 1, 6, and 20 differ from case to case: larger k's performed better in complicated relationships like checkerboard  and circular pattern, but they performed poorly comparing the adaptive approaches in linear and parabolic relationships - the two relationships are more representative of many real-world datasets than other relationships. Comparing to our adaptive selection of parameters, the KNN mutual information estimator also has the important parametric disadvantage to demand the user to specify k without any mathematical guidelines, while there is no guarantee larger k's increases the statistical power (as in sinusoidal case). As shown here with three arbitrarily set k's, they can significantly affect the power of one’s mutual information estimates, supporting the discovery of \cite{kinney2014equitability}. The adaptive MI performed slightly better than arbitrarily defined k but the overall performance is not optimal. (full results in Table S\ref{table:1DpowerNoise}).



\textbf{Robust in different sample sizes. } 
100 repetitions of observations with sample size over a 20-fold range from 20 to 400 were generated, in which the input sample was uniformed distributed on the unit interval. 
Table \ref{tab:summary} shows the average statistical power across different sample sizes for different dependence measures in the five relationships. Among all the competing measures, the proposed family of adaptive independence tests demonstrated good robustness in non-functional association patterns (ranked top 1 in all but checkerboard, and top 5 in all relationships). Comparing the three subroutines, \textit{AGTIC - s1} appears more robust than the other two. The three geo-topological transforms each have their advantages for different relationship types (full results in Table S\ref{table:1DpowerObs}).  

\textbf{Adaptive to combinatorial dependence. } 
50 repetitions of $50 \times 2$ samples were generated, such that each of the two dimensions follows either one of the 5 association patterns (linear, parabolic, sinusoidal, circular, or checkerboard) or random relationship (r), to form a combinatorial two-dimensional dependence. Table S\ref{table:2Dpower} shows the statistical power across 20 combinatorial dependence for different statistics. Among all, our methods are top 1 for all but sinusoidal-random (s-r) and checkerboard-random (k-r) relationships, and ranked among top 5 in all relationships. As expected, the statistical power in the pairs of single patterns (l-l, p-p, s-s, c-c, k-k) are higher than the pairs with different patterns, implying certain dependence interference.



\textbf{Insightful on granularity of dependence.} 
Optimal thresholds are recorded during the bivariate association experiments with increasing noise amplitude (see Table S\ref{table:optThresh} and Fig. S\ref{fig:opts} for the optimal thresholds). Fig. \ref{fig:map} reported the AGTIC maps for 16 different 1D relationships where the optimal threshold is marked the red cross. For linear or patch-like patterns, the optimal thresholds usually involved a large distance range ($u-l$ is large), while in skeleton-like patterns (e.g. spiral, circle), this range tends to be a small value, emphasizing finer structures of the data. From the grid-like structures in the maps, we can even decipher the frequency in the sin 4$\pi$ and 16$\pi$. We also noticed that similar dependencies yield similar maps. For instance, the ``Step'' and ``Exponential'' are geometrically similar despite analytically distinct. Thereby, AGTIC can help us understand the relationship in data.

\textbf{Run time analysis.} Fig. \ref{fig:time} records running time for different methods. AGTIC (in its most expensive form) has complexity $O(k^4)$, higher than dCor, but empirically we observed a lower than baseline run time when N is small (partially due to our parallel implementation). When N is large ($>10^4$), AGTIC can be expensive, for which we provided an alternative cheaper solution which replaces exhaustive searching of $k^4$ threshold sets with iterative sampling (section \ref{sec:properties}).

\textbf{Proof of concept in a real-world example.} We now apply AGTIC to the analysis of connectivity in the human brain. We analyzed functional MRI brain-activity data from a human subject (ID: 100307) of the Human Connectome Project (HCP). The cortex was parcellated into 180 regions per hemisphere using HCP multi-modal parcellation atlas (\cite{glasser2016multi} and Fig. \ref{fig:real}a), among which 22 parcels were selected as regions-of-interest (ROIs) for connectivity analysis. We computed the AGTIC for each pair of ROIs and report the statistics and their corresponding p-values (Fig. \ref{fig:real}b,c). We see blockwise mutual information among sets of brain regions. The p-value matrix of AGTIC illustrates the improved power compared to dCor in this practical application (Fig. \ref{fig:real}c,d,e). The power advantage is more formally empirically demonstrated in synthetic data using simulations, where the ground-truth is known.

\section{Conclusion}

Distance matrices capture the representational geometry and can be subjected to monotonic nonlinear transforms to capture the representational topology at different granularities. We introduced a novel family of independence tests that adapt the parameters of these geo-topological transforms so as to maximize sensitivity of the distance covariance to statistical dependency between two multivariate variables. The proposed test statistics are theoretically sound and perform well empirically, providing robust sensitivity across a wide range of univariate and multivariate relationships and across different noise levels and amounts of data. 
The present results suggest that it might be useful for a wide range of practical applications such as analyzing biological and societal data where we can (1) detect whether there is any dependency in the data and (2) understand the relationships in structured data.

\clearpage
\bibliography{main}
\bibliographystyle{unsrt}

\clearpage
\newpage

\onecolumn

\section{Multivariate Association Patterns}

\begin{figure}[h!]
\centering
    \includegraphics[width=0.235\linewidth]{Figures/Linear_data}
    \includegraphics[width=0.235\linewidth]{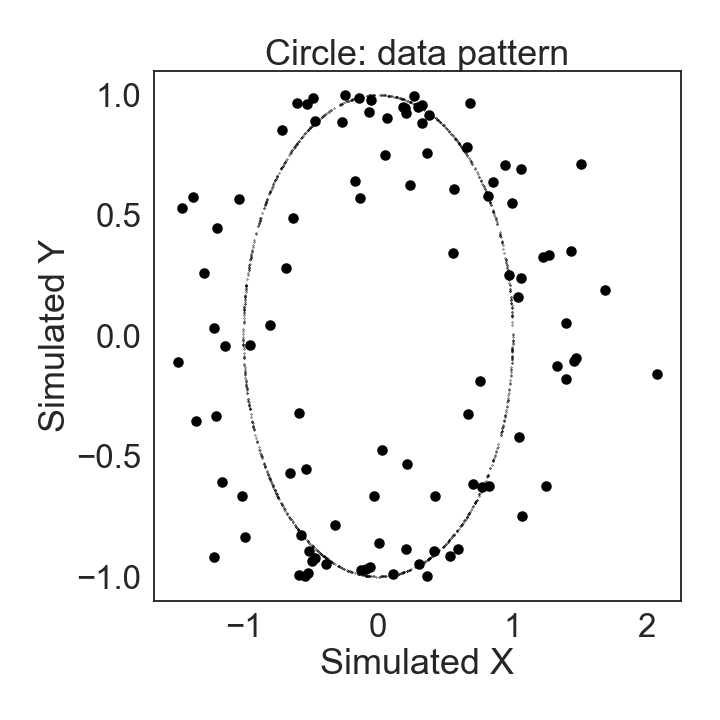}
    \includegraphics[width=0.235\linewidth]{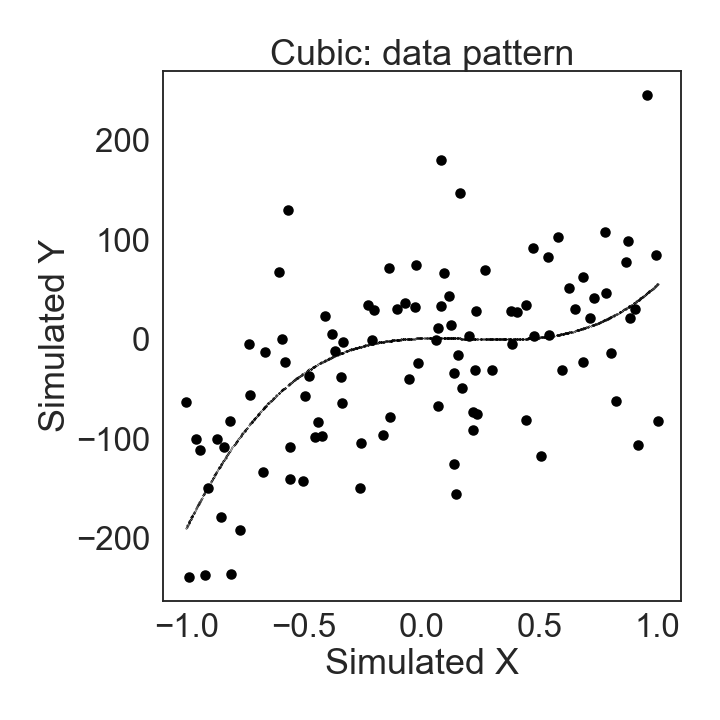}
    \includegraphics[width=0.235\linewidth]{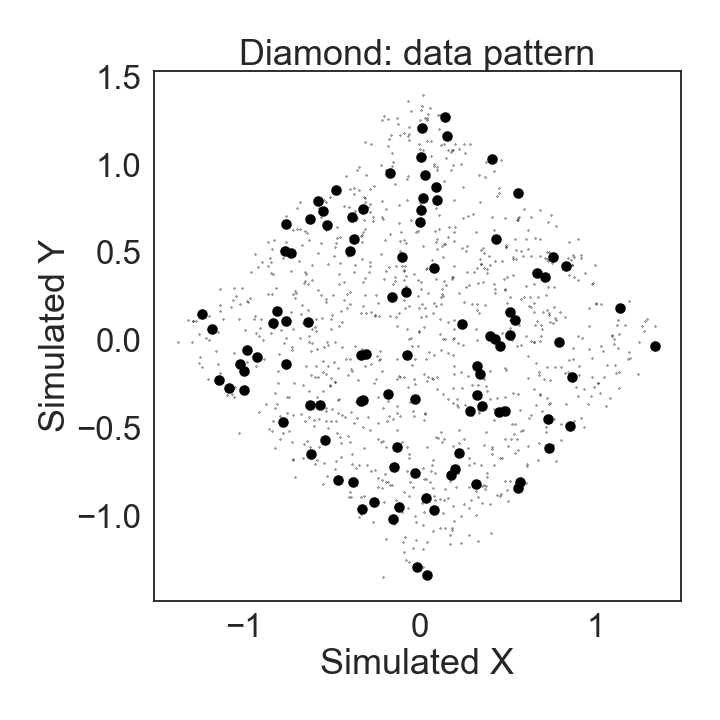}
    \includegraphics[width=0.235\linewidth]{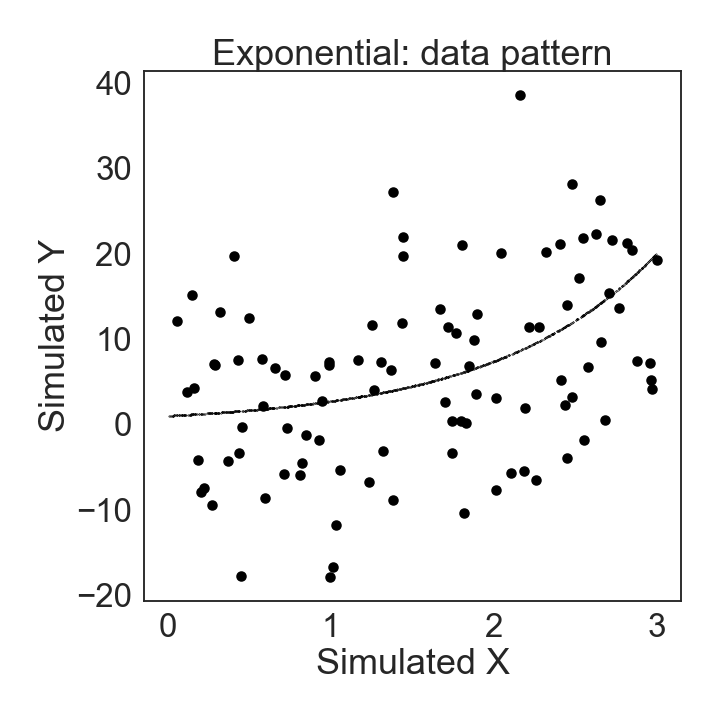}
    \includegraphics[width=0.235\linewidth]{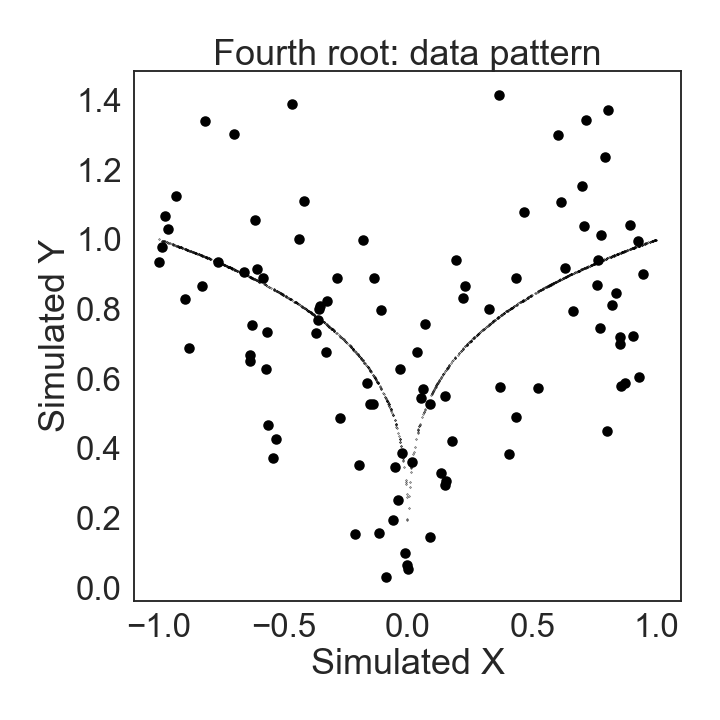}
    \includegraphics[width=0.235\linewidth]{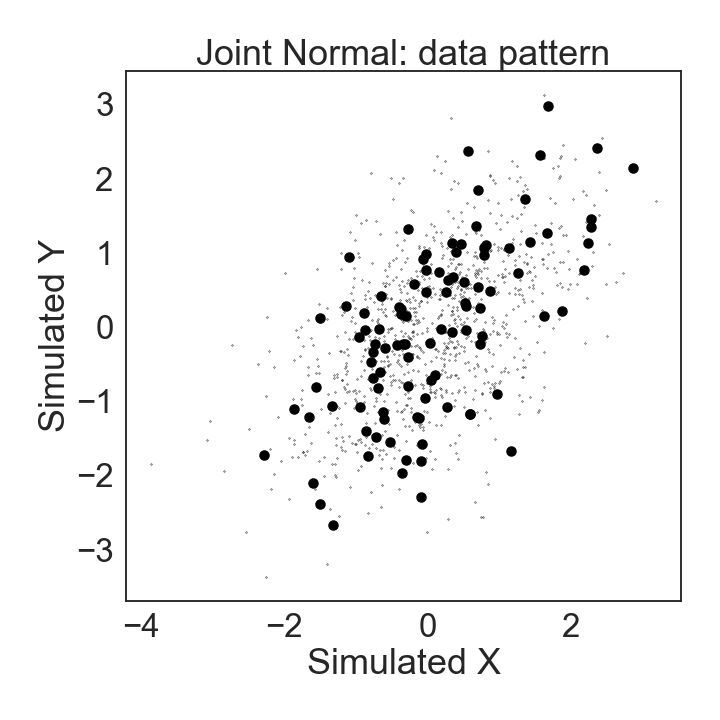}
    \includegraphics[width=0.235\linewidth]{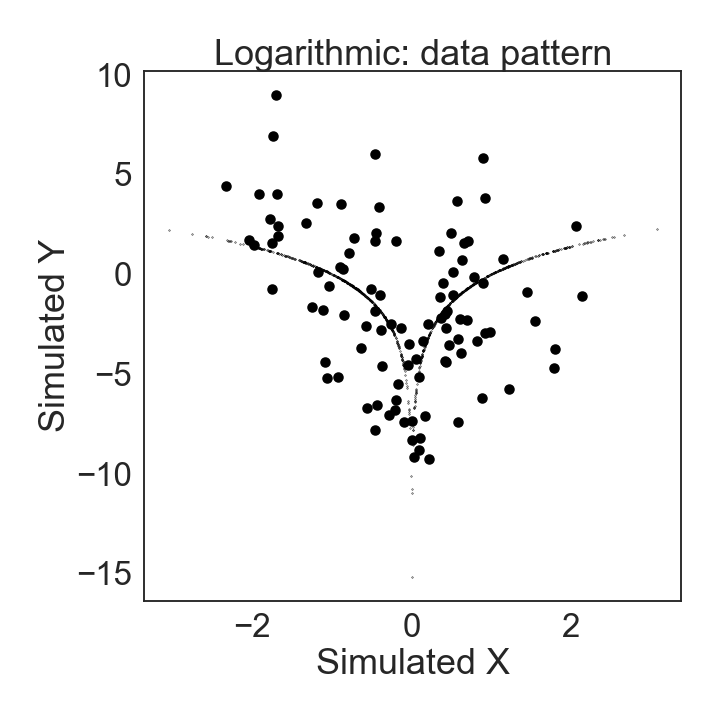}
    \includegraphics[width=0.235\linewidth]{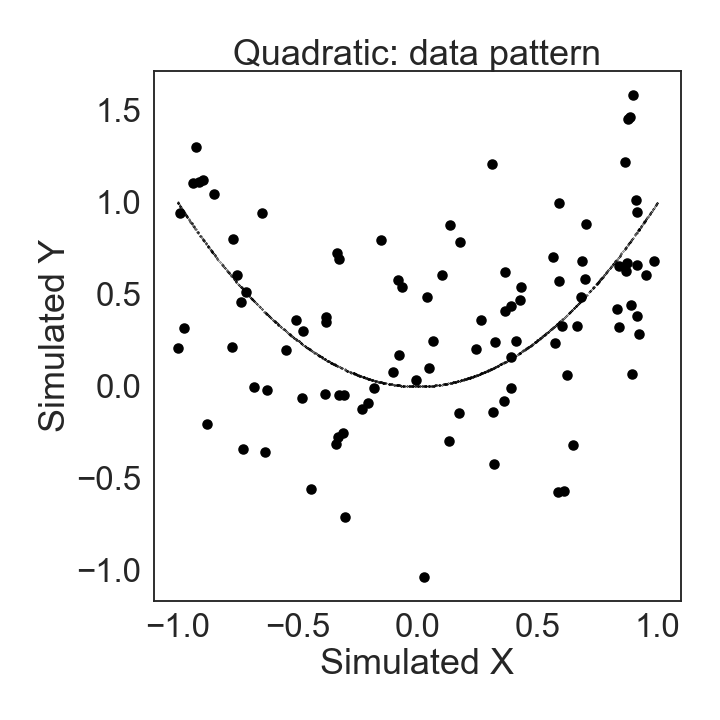}
    \includegraphics[width=0.235\linewidth]{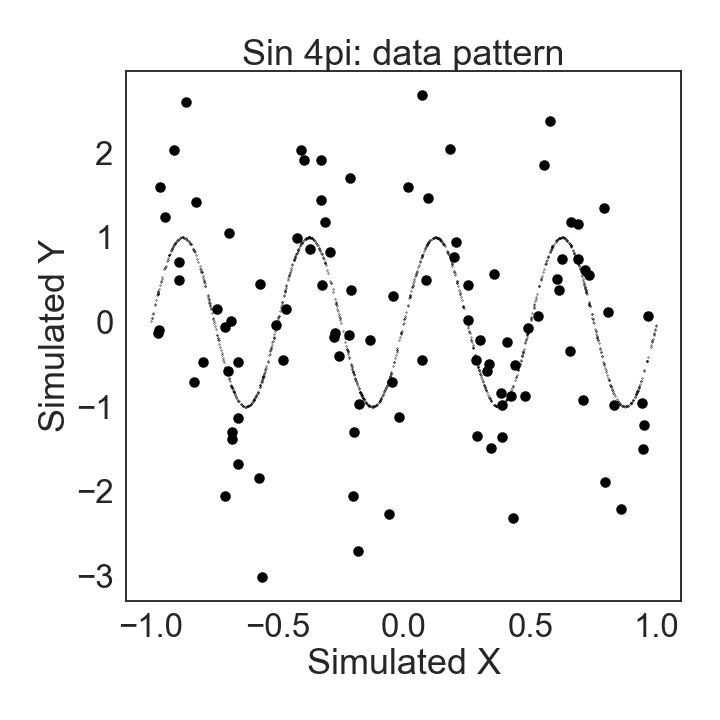}
    \includegraphics[width=0.235\linewidth]{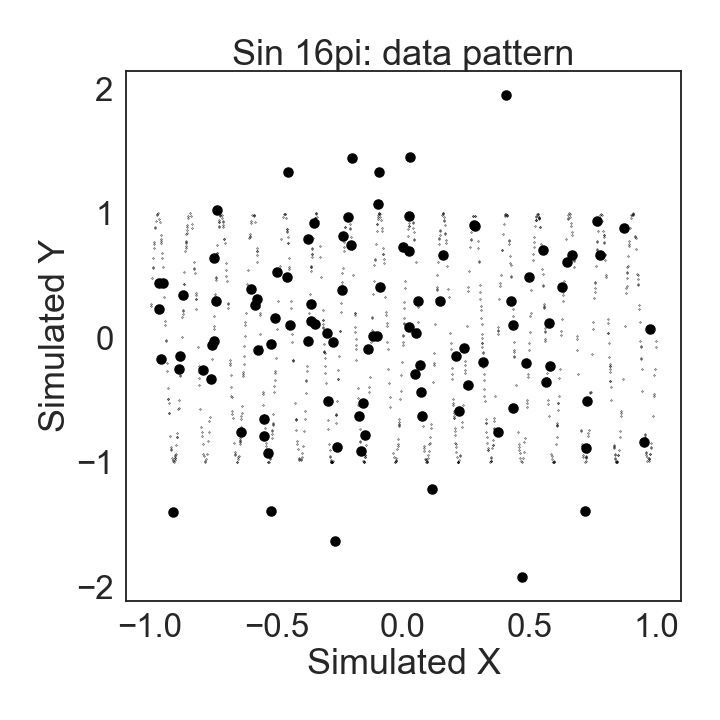}
    \includegraphics[width=0.235\linewidth]{Figures/Spiral_data}
    \includegraphics[width=0.235\linewidth]{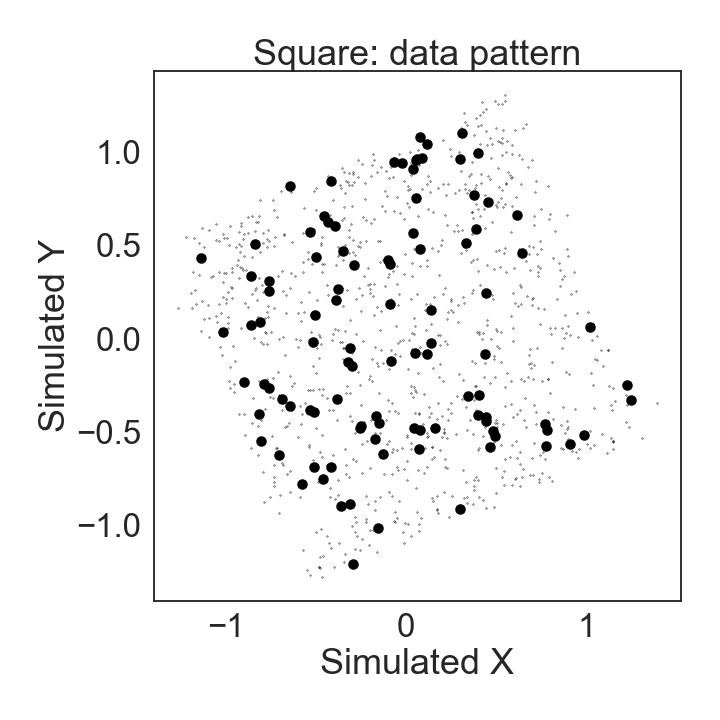}
    \includegraphics[width=0.235\linewidth]{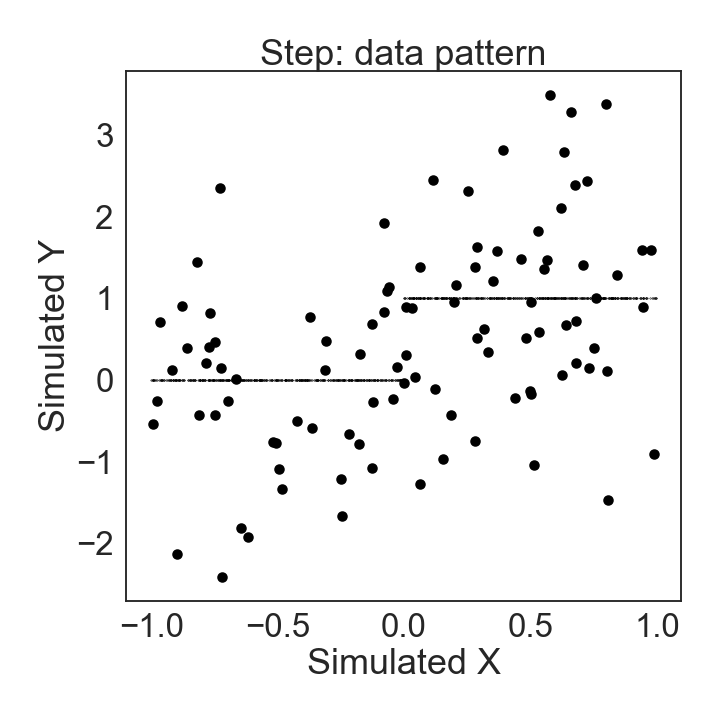}
    \includegraphics[width=0.235\linewidth]{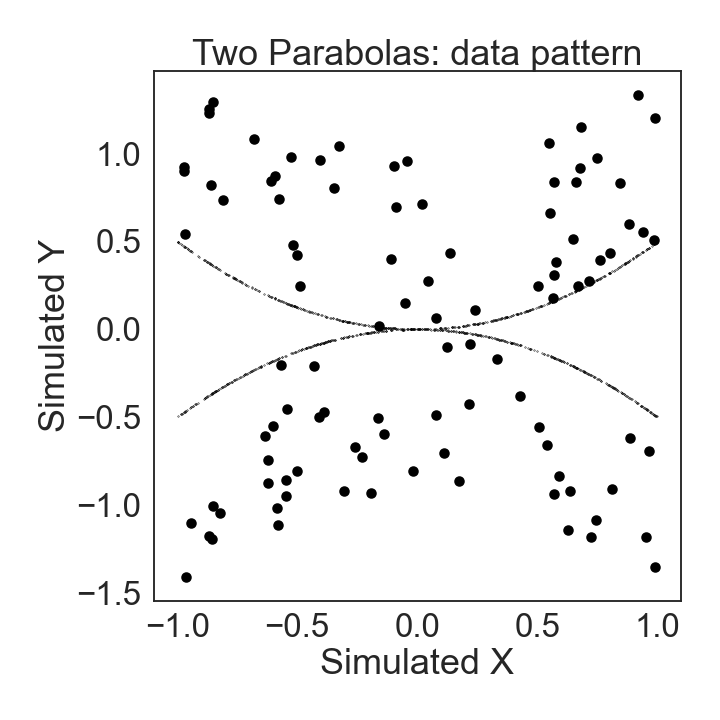}
    \includegraphics[width=0.235\linewidth]{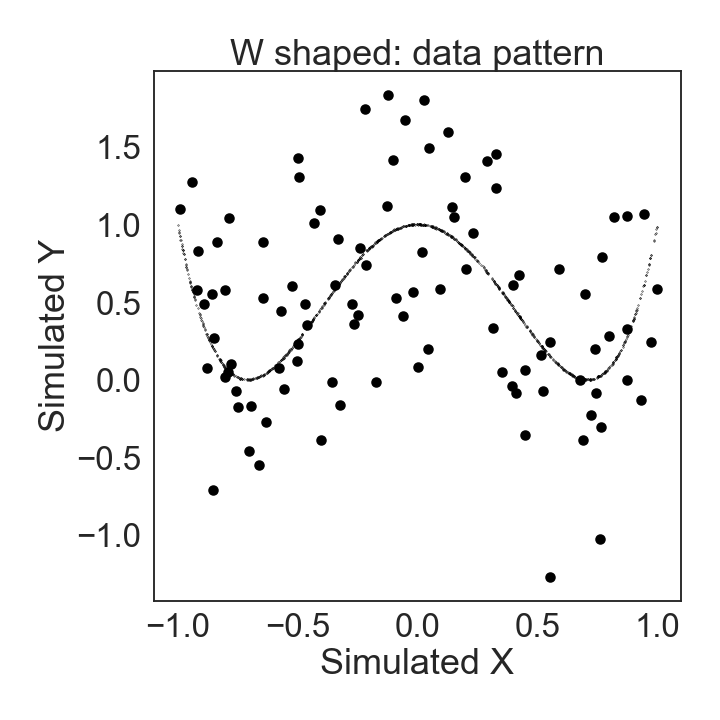}
\par\caption{\textbf{Multivariate association patterns} considered in the evaluation.}\label{fig:data}
\end{figure}

\section{Supplementary Evaluation Details}
\label{sec:paras}

\textbf{Parameter selection.}
For the family of AGTIC, the numbers of possible thresholds for the lower and upper bounds in the geo-topological transforms are set to be 5. The combinatorial search space for the boundary pairs are nchoosek$(5,2) = 10$, since the threshold search itself has a complexity of $O((k(k-1)/2)^2) \approx O(k^4)$. For HSIC, it applies a bootstrap approximation to the test threshold with kernel sizes set to the median distances for $X$ and $Y$ \cite{gretton2008kernel}. For MIC, the user-specified value $B$ was set to be $N^{0.6}$ as advocated by \cite{reshef2011detecting}. For mutual information estimator, three different k ($k=1,6,20$) were used as in \cite{kraskov2004estimating}. 

\textbf{Experimental setting.} 
In the bivariate association experiments, 50 repetitions of 200 samples were generated, in which the input sample was uniformed distributed on the unit interval. Next, we regenerated the input sample randomly in order to generate i.i.d. versions as the null distribution with equal marginals.

\section{Supplementary Figures}

\begin{figure}[!h]
\centering
    \includegraphics[width=0.42\linewidth]{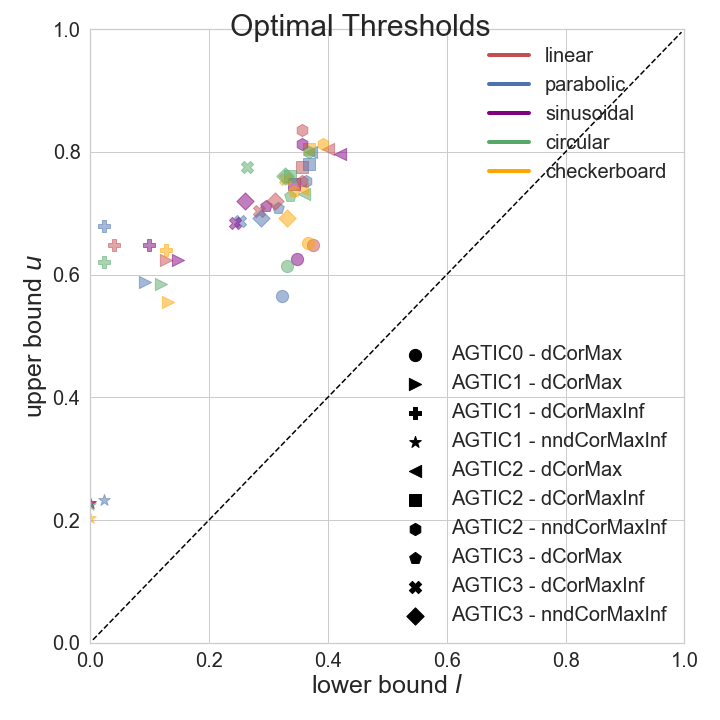}
        \includegraphics[width=0.42\linewidth]{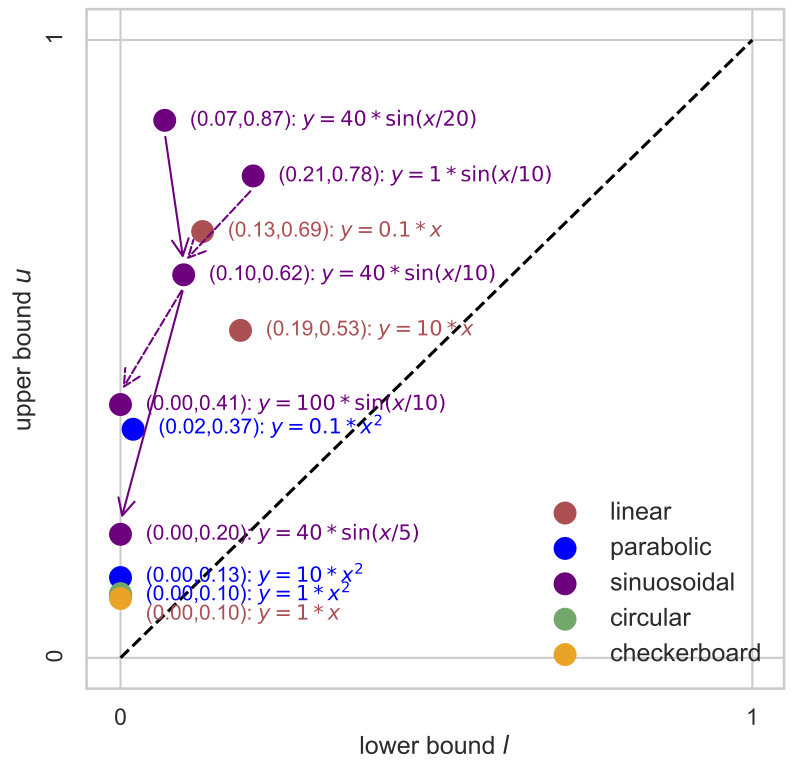}
\par\caption{\textbf{Optimal thresholds for different relationships.} Adaptively chosen lower (horizontal axis) and upper (vertical axis) bounds of the geo-topological transform for tests (shapes) and statistical patterns (colors).}\label{fig:opts}
\end{figure}

\begin{figure}[!h]
\centering
    \includegraphics[width=0.85\linewidth]{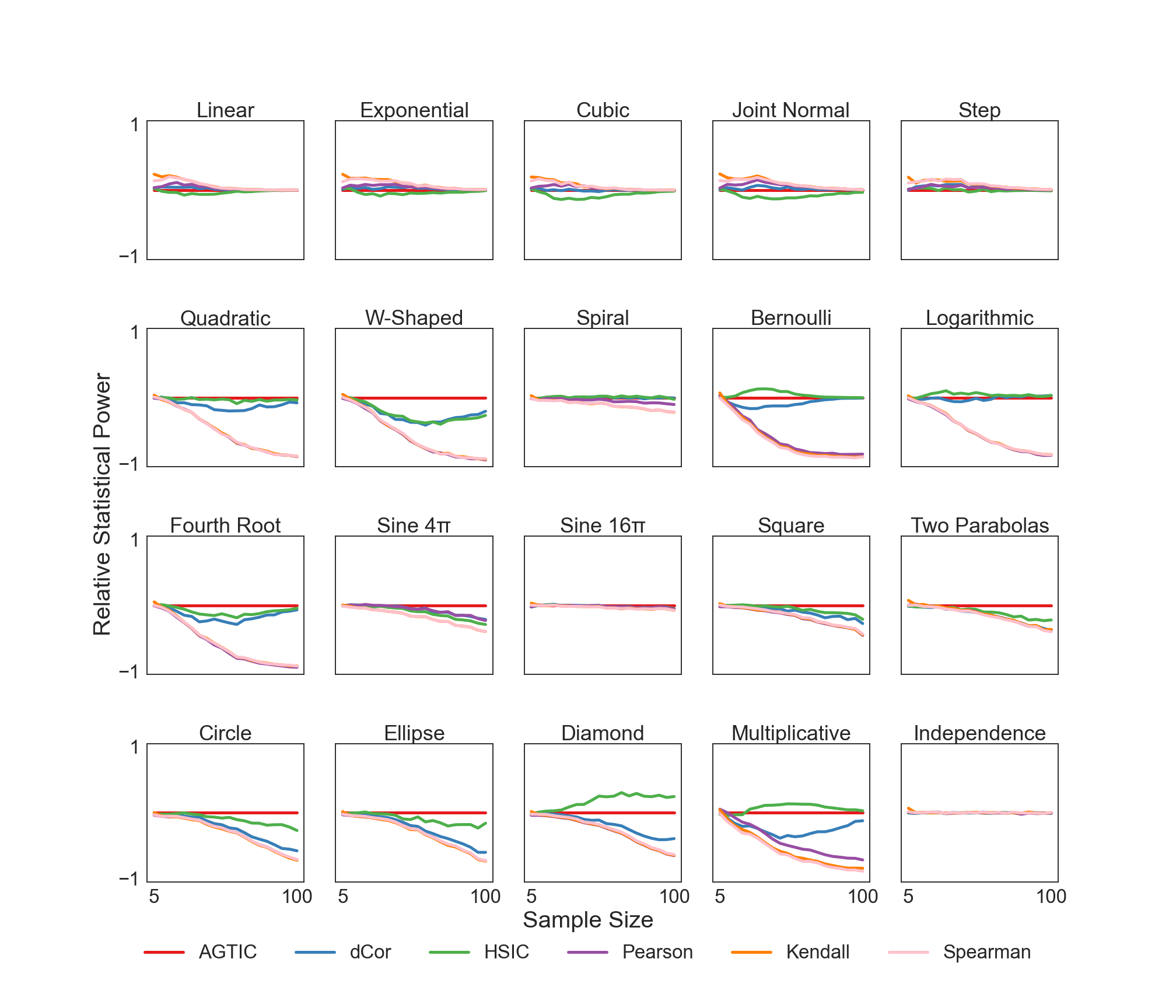}
\par\caption{\textbf{Relative statistical power over sample size} in 20 multivariate relationships.}\label{fig:smp}
         \vspace{-0.2in}
\end{figure}

\begin{figure}[!h]
\centering
    \includegraphics[width=0.8\linewidth]{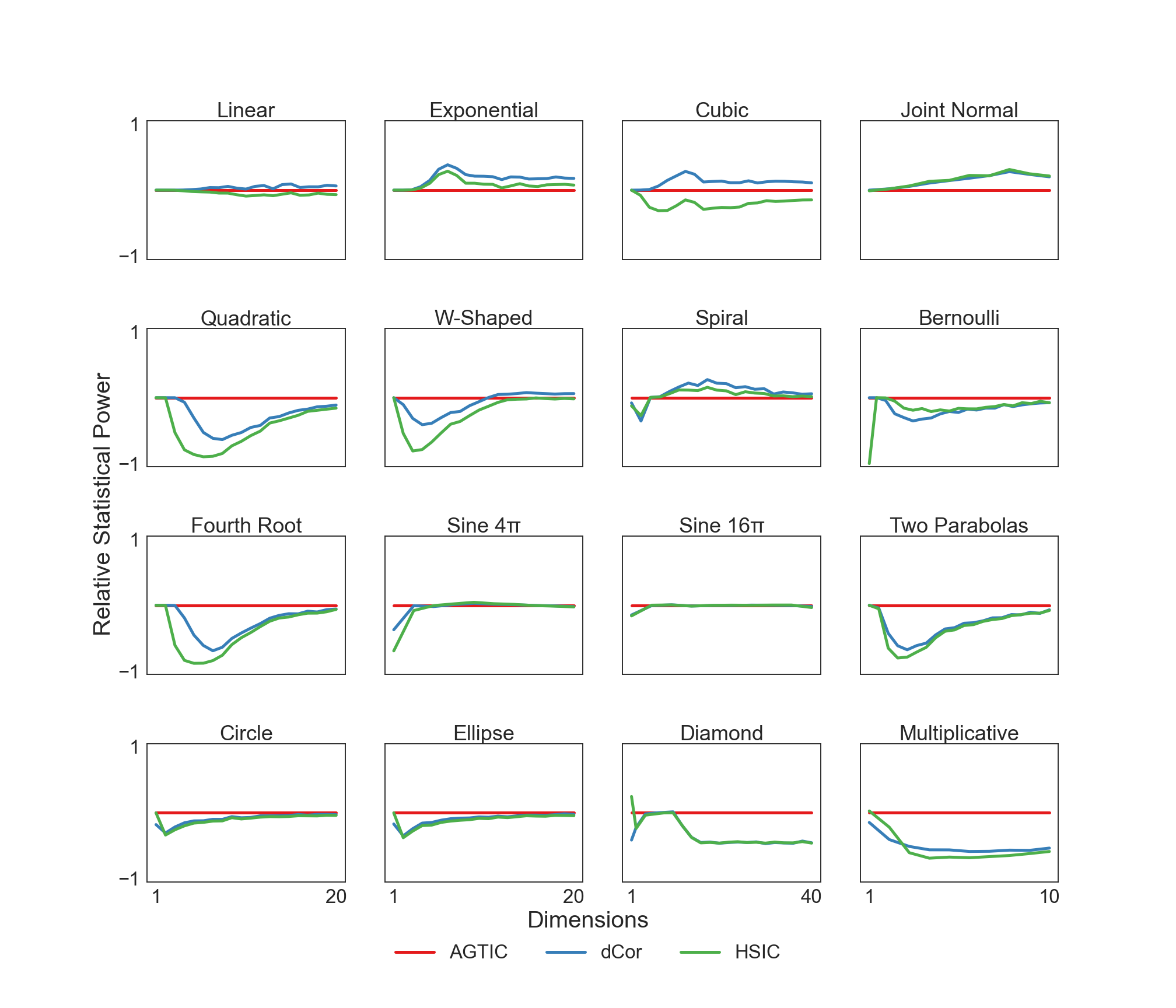}
\par\caption{\textbf{Relative statistical power over dimension size} in 16 multivariate relationships.}\label{fig:dim}
         \vspace{-0.2in}
\end{figure}

\begin{figure}[!h]
\centering
    \includegraphics[width=0.5\linewidth]{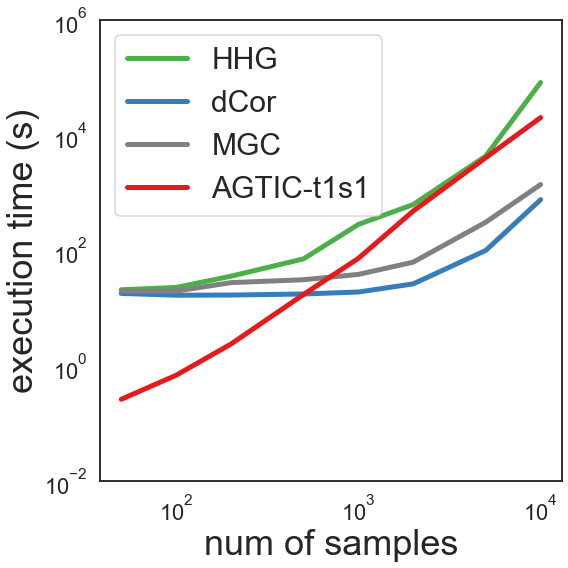}
\par\caption{\textbf{Execution wall time} of different methods (AGTIC in its most time consuming variant).}\label{fig:time}
\end{figure}

\clearpage
\section{Supplementary Tables}
\label{sec:addtable}

\begin{table*}[tbh]
	\centering
	\caption{Power of different tests (rows) for detecting relationships (columns) between two univariate variables (averaged over different noise amplitudes, rows ranked by average power)}
	\resizebox{1\linewidth}{!}{
		\begin{tabular}{ l | l | l | l | l | l | l }
			 &  linear & parabolic & sinusoidal & circular & checkerboard & average \\ 
			 \thickhline
pAGTIC - t1s1	&	0.594	$\pm$	0.400	&	0.534	$\pm$	0.431	&	\textbf{0.712}	$\pm$	\textbf{0.384}	&	0.680	$\pm$	0.414	&	\textbf{0.618}	$\pm$	\textbf{0.410}	&	\textbf{0.628}	$\pm$	\textbf{0.396}	\\
MI (k=20)	&	0.502	$\pm$	0.404	&	0.534	$\pm$	0.354	&	0.504	$\pm$	0.409	&	0.702	$\pm$	0.391	&	0.582	$\pm$	0.413	&	0.565	$\pm$	0.386	\\
AGTIC - t1s1	&	0.668	$\pm$	0.349	&	0.580	$\pm$	0.428	&	0.552	$\pm$	0.452	&	0.590	$\pm$	0.464	&	0.420	$\pm$	0.324	&	0.562	$\pm$	0.399	\\
pAGTIC - t2s1	&	0.484	$\pm$	0.399	&	0.366	$\pm$	0.374	&	0.616	$\pm$	0.415	&	0.688	$\pm$	0.414	&	0.616	$\pm$	0.435	&	0.554	$\pm$	0.408	\\
AGTIC - t3s1	&	0.484	$\pm$	0.437	&	0.448	$\pm$	0.433	&	0.594	$\pm$	0.430	&	0.624	$\pm$	0.423	&	0.606	$\pm$	0.441	&	0.551	$\pm$	0.421	\\
pAGTIC - t1s3	&	0.578	$\pm$	0.399	&	0.528	$\pm$	0.401	&	0.624	$\pm$	0.394	&	0.604	$\pm$	0.389	&	0.396	$\pm$	0.358	&	0.546	$\pm$	0.381	\\
AGTIC - t1s2	&	0.674	$\pm$	0.351	&	0.582	$\pm$	0.399	&	0.546	$\pm$	0.422	&	0.502	$\pm$	0.480	&	0.408	$\pm$	0.321	&	0.542	$\pm$	0.392	\\
AGTIC - t3s2	&	0.632	$\pm$	0.371	&	0.580	$\pm$	0.431	&	0.542	$\pm$	0.449	&	0.512	$\pm$	0.468	&	0.432	$\pm$	0.430	&	0.540	$\pm$	0.419	\\
AGTIC - t0s2	&	0.374	$\pm$	0.350	&	0.546	$\pm$	0.393	&	0.578	$\pm$	0.423	&	0.656	$\pm$	0.433	&	0.532	$\pm$	0.440	&	0.537	$\pm$	0.403	\\
AGTIC - t2s2	&	0.636	$\pm$	0.382	&	0.578	$\pm$	0.427	&	0.500	$\pm$	0.432	&	0.540	$\pm$	0.466	&	0.380	$\pm$	0.364	&	0.527	$\pm$	0.408	\\
pAGTIC - t1s2	&	0.636	$\pm$	0.371	&	\textbf{0.602}	$\pm$	\textbf{0.373}	&	0.600	$\pm$	0.396	&	0.560	$\pm$	0.449	&	0.230	$\pm$	0.177	&	0.526	$\pm$	0.381	\\
AGTIC - t3s3	&	0.578	$\pm$	0.387	&	0.560	$\pm$	0.432	&	0.526	$\pm$	0.402	&	0.514	$\pm$	0.451	&	0.448	$\pm$	0.398	&	0.525	$\pm$	0.400	\\
pAGTIC - t2s3	&	0.562	$\pm$	0.404	&	0.494	$\pm$	0.419	&	0.538	$\pm$	0.400	&	0.582	$\pm$	0.422	&	0.440	$\pm$	0.397	&	0.523	$\pm$	0.395	\\
AGTIC - t2s1	&	0.566	$\pm$	0.418	&	0.546	$\pm$	0.439	&	0.500	$\pm$	0.447	&	0.614	$\pm$	0.453	&	0.376	$\pm$	0.333	&	0.520	$\pm$	0.411	\\
MI (k=6)	&	0.380	$\pm$	0.360	&	0.406	$\pm$	0.375	&	0.586	$\pm$	0.381	&	0.620	$\pm$	0.413	&	0.602	$\pm$	0.423	&	0.519	$\pm$	0.389	\\
AGTIC - t1s3	&	0.656	$\pm$	0.337	&	0.582	$\pm$	0.408	&	0.494	$\pm$	0.433	&	0.536	$\pm$	0.461	&	0.318	$\pm$	0.207	&	0.517	$\pm$	0.382	\\
pAGTIC - t3s3	&	0.536	$\pm$	0.424	&	0.442	$\pm$	0.400	&	0.570	$\pm$	0.418	&	0.594	$\pm$	0.417	&	0.442	$\pm$	0.380	&	0.517	$\pm$	0.396	\\
pAGTIC - t2s2	&	0.624	$\pm$	0.375	&	0.504	$\pm$	0.416	&	0.542	$\pm$	0.391	&	0.530	$\pm$	0.424	&	0.214	$\pm$	0.168	&	0.483	$\pm$	0.379	\\
pAGTIC - t3s2	&	0.604	$\pm$	0.392	&	0.514	$\pm$	0.416	&	0.498	$\pm$	0.409	&	0.536	$\pm$	0.450	&	0.250	$\pm$	0.215	&	0.480	$\pm$	0.389	\\
AGTIC - t2s3	&	0.536	$\pm$	0.421	&	0.564	$\pm$	0.415	&	0.410	$\pm$	0.418	&	0.540	$\pm$	0.463	&	0.340	$\pm$	0.278	&	0.478	$\pm$	0.397	\\
dCor	&	0.676	$\pm$	0.353	&	0.550	$\pm$	0.425	&	0.452	$\pm$	0.424	&	0.446	$\pm$	0.448	&	0.210	$\pm$	0.140	&	0.467	$\pm$	0.392	\\
Adaptive MI - s2	&	0.468	$\pm$	0.396	&	0.386	$\pm$	0.355	&	0.354	$\pm$	0.403	&	0.634	$\pm$	0.437	&	0.428	$\pm$	0.324	&	0.454	$\pm$	0.382	\\
Hoeffding's D	&	0.650	$\pm$	0.356	&	0.460	$\pm$	0.414	&	0.460	$\pm$	0.431	&	0.498	$\pm$	0.475	&	0.200	$\pm$	0.112	&	0.454	$\pm$	0.393	\\
HSIC	&	0.504	$\pm$	0.416	&	0.556	$\pm$	0.363	&	0.324	$\pm$	0.369	&	0.670	$\pm$	0.439	&	0.196	$\pm$	0.082	&	0.450	$\pm$	0.383	\\
MIC	&	0.344	$\pm$	0.335	&	0.378	$\pm$	0.299	&	0.586	$\pm$	0.288	&	0.438	$\pm$	0.366	&	0.310	$\pm$	0.193	&	0.411	$\pm$	0.305	\\
AGTIC - t0s1	&	0.368	$\pm$	0.383	&	0.446	$\pm$	0.367	&	0.392	$\pm$	0.414	&	0.492	$\pm$	0.504	&	0.264	$\pm$	0.326	&	0.392	$\pm$	0.394	\\
Adaptive MI - s1	&	0.170	$\pm$	0.206	&	0.262	$\pm$	0.314	&	0.454	$\pm$	0.394	&	0.494	$\pm$	0.424	&	0.482	$\pm$	0.449	&	0.372	$\pm$	0.377	\\
rdmCor	&	0.426	$\pm$	0.406	&	0.534	$\pm$	0.420	&	0.028	$\pm$	0.023	&	\textbf{0.728}	$\pm$	\textbf{0.408}	&	0.036	$\pm$	0.042	&	0.350	$\pm$	0.415	\\
MI (k=1)	&	0.174	$\pm$	0.107	&	0.208	$\pm$	0.228	&	0.402	$\pm$	0.379	&	0.448	$\pm$	0.415	&	0.396	$\pm$	0.390	&	0.326	$\pm$	0.332	\\
pAGTIC - t3s1	&	0.160	$\pm$	0.263	&	0.054	$\pm$	0.040	&	0.344	$\pm$	0.401	&	0.310	$\pm$	0.363	&	0.316	$\pm$	0.330	&	0.237	$\pm$	0.315	\\
R$^2$	&	\textbf{0.710}	$\pm$	\textbf{0.325}	&	0.136	$\pm$	0.104	&	0.054	$\pm$	0.053	&	0.028	$\pm$	0.036	&	0.072	$\pm$	0.043	&	0.200	$\pm$	0.300	\\  
            \hline
		\end{tabular}
	}  
\label{table:1DpowerNoise}
	     \vspace{-0.05in}
\end{table*}

\begin{table}[tbhp]
	\centering
	\caption{80\% power noise level for 1-d data (ranked by avg. power)}
	\resizebox{0.8\linewidth}{!}{
		\begin{tabular}{ l | c | c | c | c | c  }
			 &  linear & parabolic & sinusoidal & circular & checkerboard  \\ \thickhline
pAGTIC - t1s1	&	2.657	&	2.411	&	\textbf{3.975}	&	3.856	&	3.081	\\
MI (k=20)	&	2.302	&	2.289	&	2.191	&	3.987	&	2.839	\\
AGTIC - t1s1	&	2.945	&	2.922	&	2.485	&	3.212	&	1.000	\\
pAGTIC - t2s1	&	1.972	&	1.507	&	3.087	&	4.013	&	\textbf{3.282}	\\
AGTIC - t3s1	&	2.249	&	2.015	&	2.985	&	3.130	&	3.212	\\
pAGTIC - t1s3	&	2.611	&	2.296	&	3.036	&	2.154	&	1.347	\\
AGTIC - t1s2	&	3.032	&	2.573	&	2.305	&	2.720	&	1.000	\\
AGTIC - t3s2	&	2.626	&	\textbf{2.955}	&	2.400	&	2.603	&	2.175	\\
AGTIC - t0s2	&	1.435	&	2.154	&	2.783	&	3.594	&	2.594	\\
AGTIC - t2s2	&	2.713	&	2.895	&	2.182	&	2.907	&	1.417	\\
pAGTIC - t1s2	&	2.713	&	2.434	&	2.524	&	2.945	&	1.000	\\
AGTIC - t3s3	&	2.399	&	2.864	&	2.073	&	2.440	&	1.960	\\
pAGTIC - t2s3	&	2.394	&	2.304	&	2.845	&	2.110	&	2.195	\\
AGTIC - t2s1	&	2.626	&	2.852	&	2.364	&	3.362	&	1.000	\\
MI (k=6)	&	1.423	&	1.463	&	2.856	&	3.081	&	3.239	\\
AGTIC - t1s3	&	2.668	&	2.812	&	2.307	&	2.837	&	1.000	\\
pAGTIC - t3s3	&	2.554	&	1.911	&	2.860	&	2.994	&	1.801	\\
pAGTIC - t2s2	&	2.845	&	2.346	&	2.419	&	1.954	&	1.000	\\
pAGTIC - t3s2	&	2.837	&	2.337	&	2.272	&	2.783	&	1.000	\\
AGTIC - t2s3	&	2.434	&	2.854	&	1.830	&	2.864	&	1.080	\\
dCor	&	3.188	&	2.833	&	1.911	&	2.201	&	1.000	\\
Adaptive MI - s2	&	1.960	&	1.243	&	1.377	&	3.774	&	1.044	\\
Hoeffding's D	&	2.783	&	2.244	&	2.057	&	2.573	&	1.000	\\
HSIC	&	2.265	&	2.280	&	1.406	&	3.987	&	1.000	\\
MIC	&	1.000	&	1.219	&	1.000	&	1.000	&	1.000	\\
AGTIC - t0s1	&	1.531	&	1.448	&	1.770	&	2.822	&	1.204	\\
Adaptive MI - s1	&	1.000	&	1.146	&	1.668	&	2.224	&	2.355	\\
rdmCor	&	1.801	&	2.154	&	1.000	&	\textbf{4.706}	&	1.000	\\
MI (k=1)	&	1.000	&	1.000	&	1.565	&	2.022	&	1.507	\\
pAGTIC - t3s1	&	1.036	&	1.000	&	1.561	&	1.390	&	1.073	\\
R$^2$	&	\textbf{3.233}	&	1.000	&	1.000	&	1.000	&	1.000	\\           
            \hline
		\end{tabular}
	}  
\label{table:1DpowerStaircase}
\end{table}

\begin{table*}[h!]
	\centering
	\caption{Two-dimensional data with paired dependency (ranked by average power)}
	\resizebox{1\linewidth}{!}{
		\begin{tabular}{ l | l | l | l | l | l | l | l | l | l | l | l | l | l | l | l | l | l | l | l | l | l }
			 &  l-l & l-p & l-s & l-c & l-k & l-r & p-p & p-s & p-c & p-k & p-r & s-s & s-c & s-k & s-r & c-c & c-k & c-r & k-k & k-r & average \\ \thickhline
            MI (k=6) &  0.88 & 0.98 & 0.98 & 0.88 & \textbf{1.00} & 0.96 & 0.82 & 0.94 & 0.98 & 0.92 & 0.88 & 0.92 & \textbf{1.00} & 0.98 & \textbf{0.98} & 0.94 & 0.88 & 0.94 & 0.98 & \textbf{0.94} & 
\textbf{0.939} $\pm$ \textbf{0.050}       \\
            MI (k=1) & 0.82 & \textbf{1.00} & 0.92 & 0.90 & 0.96 & 0.98 & 0.84 & 0.96 & 0.98 & 0.88 & 0.88 & 0.94 & 0.98 & 0.96 & 0.94 & 0.92 & 0.96 & 0.90 & 0.84 & 0.00 & 
0.878 $\pm$ 0.213     \\
            AGTIC - t3s1 & \textbf{1.00} & \textbf{1.00} & \textbf{1.00} & \textbf{1.00} & \textbf{1.00} & \textbf{1.00} & \textbf{1.00} & \textbf{1.00} & \textbf{1.00} & \textbf{1.00} & \textbf{1.00} & \textbf{1.00} & \textbf{1.00} & 0.00 & 0.00 & \textbf{1.00} & 0.00 & \textbf{1.00} & \textbf{1.00} & 0.00 & 
0.800 $\pm$ 0.410        \\
            Hoeffding's D & \textbf{1.00} & \textbf{1.00} & \textbf{1.00} & \textbf{1.00} & \textbf{1.00} & \textbf{1.00} & \textbf{1.00} & \textbf{1.00} & \textbf{1.00} & \textbf{1.00} & \textbf{1.00} & \textbf{1.00} & \textbf{1.00} & 0.00 & 0.00 & \textbf{1.00} & 0.00 & 0.00 & \textbf{1.00} & 0.00 & 
0.750 $\pm$ 0.444        \\
            MIC & 0.94 & 0.96 & 0.90 & 0.02 & 0.80 & 0.96 & 0.98 & 0.98 & \textbf{1.00} & 0.92 & 0.98 & \textbf{1.00} & 0.80 & 0.00 & 0.00 & 0.90 & 0.64 & 0.12 & 0.08 & 0.00 & 0.649 $\pm$ 0.421     \\
            MI (k=20) & 0.86 & \textbf{1.00} & 0.00 & 0.92 & 0.96 & 0.98 & 0.86 & 0.98 & \textbf{1.00} & 0.88 & 0.86 & \textbf{1.00} & 0.02 & 0.02 & 0.02 & 0.92 & 0.00 & 0.94 & 0.14 & 0.00 & 
0.618 $\pm$ 0.447      \\
            HSIC & \textbf{1.00} & \textbf{1.00} & \textbf{1.00} & 0.00 & 0.00 & \textbf{1.00} & \textbf{1.00} & 0.00 & 0.00 & 0.00 & \textbf{1.00} & \textbf{1.00} & \textbf{1.00} & 0.00 & 0.00 & \textbf{1.00} & 0.00 & \textbf{1.00} & \textbf{1.00} & 0.00 & 
0.550 $\pm$ 0.510          \\
            AGTIC - t0s2 & \textbf{1.00} & \textbf{1.00} & 0.12 & 0.34 & 0.04 & 0.74 & \textbf{1.00} & 0.16 & 0.48 & 0.04 & 0.56 & \textbf{1.00} & \textbf{1.00} & 0.02 & 0.00 & \textbf{1.00} & 0.00 & 0.00 & \textbf{1.00} & 0.00 & 
0.475 $\pm$ 0.444           \\
            dCor &  \textbf{1.00} & \textbf{1.00} & \textbf{1.00} & 0.00 & 0.00 & \textbf{1.00} & \textbf{1.00} & 0.00 & 0.00 & 0.00 & \textbf{1.00} & \textbf{1.00} & 0.00 & 0.00 & 0.00 & \textbf{1.00} & 0.00 & 0.00 & \textbf{1.00} & 0.00 &  
0.450 $\pm$ 0.510      \\
            AGTIC - t1s2 & \textbf{1.00} & \textbf{1.00} & 0.42 & 0.00 & 0.02 & 0.62 & \textbf{1.00} & 0.00 & 0.00 & 0.02 & 0.68 & \textbf{1.00} & 0.02 & 0.00 & 0.74 & \textbf{1.00} & 0.00 & 0.00 & \textbf{1.00} & 0.00 &
 0.426 $\pm$ 0.456         \\
0.423 $\pm$ 0.469            \\
            AGTIC - t2s1 & \textbf{1.00} & \textbf{1.00} & 0.00 & 0.00 & 0.00 & \textbf{1.00} & \textbf{1.00} & 0.00 & 0.00 & 0.00 & \textbf{1.00} & \textbf{1.00} & 0.00 & 0.00 & 0.00 & \textbf{1.00} & 0.00 & 0.00 & \textbf{1.00} & 0.00 & 
0.400 $\pm$ 0.503           \\
            AGTIC - t1s1 & \textbf{1.00} & \textbf{1.00} & 0.48 & 0.00 & 0.00 & 0.00 & \textbf{1.00} & 0.00 & 0.00 & 0.00 & \textbf{1.00} & \textbf{1.00} & 0.00 & 0.00 & 0.00 & \textbf{1.00} & 0.00 & 0.00 & \textbf{1.00} & 0.00 & 
0.374 $\pm$ 0.483       \\
            AGTIC - t3s3 & \textbf{1.00} & \textbf{1.00} & 0.20 & 0.00 & 0.14 & 0.04 & \textbf{1.00} & 0.36 & 0.04 & 0.06 & 0.00 & \textbf{1.00} & 0.08 & 0.00 & 0.06 & \textbf{1.00} & 0.06 & 0.02 & 0.78 & 0.12 & 
0.348 $\pm$ 0.424          \\
            AGTIC - t1s3 & \textbf{1.00} & \textbf{1.00} & 0.12 & 0.02 & 0.16 & 0.12 & \textbf{1.00} & 0.02 & 0.04 & 0.04 & 0.12 & \textbf{1.00} & 0.10 & 0.02 & 0.02 & 0.98 & 0.10 & 0.08 & 0.92 & 0.06 & 
0.346 $\pm$ 0.430         \\
            AGTIC - t2s3 & \textbf{1.00} & \textbf{1.00} & 0.36 & 0.04 & 0.10 & 0.22 & \textbf{1.00} & 0.00 & 0.08 & 0.04 & 0.18 & \textbf{1.00} & 0.02 & 0.00 & 0.02 & 0.98 & 0.08 & 0.00 & 0.68 & 0.12 & 
0.346 $\pm$ 0.415          \\
            AGTIC - t3s2 & \textbf{1.00} & \textbf{1.00} & 0.70 & 0.00 & 0.02 & 0.04 & \textbf{1.00} & 0.00 & 0.00 & 0.00 & 0.06 & \textbf{1.00} & 0.00 & 0.00 & 0.00 & \textbf{1.00} & 0.00 & 0.00 & \textbf{1.00} & 0.00 & 
0.341 $\pm$ 0.468          \\
            AGTIC - t2s2 & \textbf{1.00} & \textbf{1.00} & 0.52 & 0.00 & 0.00 & 0.04 & \textbf{1.00} & 0.00 & 0.00 & 0.00 & 0.04 & \textbf{1.00} & 0.00 & 0.00 & 0.02 & \textbf{1.00} & 0.02 & 0.00 & \textbf{1.00} & 0.00 & 
0.332 $\pm$ 0.463          \\
            AGTIC - t0s1 & \textbf{1.00} & \textbf{1.00} & 0.00 & 0.04 & 0.04 & 0.18 & \textbf{1.00} & 0.00 & 0.10 & 0.00 & 0.00 & \textbf{1.00} & 0.00 & 0.00 & 0.00 & \textbf{1.00} & 0.00 & 0.00 & \textbf{1.00} & 0.00 & 
0.318 $\pm$ 0.460           \\
            rdmCor & \textbf{1.00} & \textbf{1.00} & 0.00 & 0.00 & 0.00 & \textbf{1.00} & \textbf{1.00} & 0.00 & 0.00 & 0.00 & \textbf{1.00} & 0.00 & 0.00 & 0.00 & 0.00 & 0.00 & 0.00 & 0.00 & 0.00 & 0.00 & 
0.250 $\pm$ 0.444           \\
            R$^2$ & \textbf{1.00} & \textbf{1.00} & 0.00 & 0.00 & 0.00 & 0.00 & \textbf{1.00} & 0.00 & 0.00 & 0.00 & 0.00 & 0.00 & 0.00 & 0.00 & 0.00 & 0.00 & 0.00 & 0.00 & \textbf{1.00} & 0.00 &  
0.200 $\pm$ 0.410      \\
            \hline
		\end{tabular}
	}  
\label{table:2Dpower}
\end{table*}

\begin{table*}[tbh]
	\centering
	\caption{One-dimensional data over different sample sizes (ranked by average power)}
	\resizebox{1\linewidth}{!}{
		\begin{tabular}{ l | l | l | l | l | l | l }
			 &  linear & parabolic & sinusoidal & circular & checkerboard & average \\ \thickhline
AGTIC - t3s1 & \textbf{1.000} $\pm$ \textbf{0.000} & \textbf{1.000} $\pm$ \textbf{0.000} & \textbf{1.000} $\pm$ \textbf{0.000} & \textbf{1.000} $\pm$ \textbf{0.000} & 0.950 $\pm$ 0.218 & \textbf{0.990} $\pm$ \textbf{0.022} \\
MI (k=1) & 0.995 $\pm$ 0.011 & 0.985 $\pm$ 0.019 & 0.991 $\pm$ 0.015 & 0.993 $\pm$ 0.015 & \textbf{0.983} $\pm$ \textbf{0.018} & 0.989 $\pm$ 0.005 \\
AGTIC - t0s2 & \textbf{1.000} $\pm$ \textbf{0.000} & \textbf{1.000} $\pm$ \textbf{0.000} & 0.967 $\pm$ 0.144 & \textbf{1.000} $\pm$ \textbf{0.000} & 0.928 $\pm$ 0.230 & 0.979 $\pm$ 0.032 \\
AGTIC - t0s1 & \textbf{1.000} $\pm$ \textbf{0.000} & \textbf{1.000} $\pm$ \textbf{0.000} & 0.950 $\pm$ 0.218 & \textbf{1.000} $\pm$ \textbf{0.000} & 0.908 $\pm$ 0.279 & 0.972 $\pm$ 0.042 \\
MI (k=6) & 0.992 $\pm$ 0.018 & 0.982 $\pm$ 0.024 & 0.939 $\pm$ 0.211 & 0.995 $\pm$ 0.010 & 0.942 $\pm$ 0.216 & 0.970 $\pm$ 0.027 \\
AGTIC - t3s3 & 0.998 $\pm$ 0.009 & \textbf{1.000} $\pm$ \textbf{0.002} & 0.969 $\pm$ 0.135 & 0.986 $\pm$ 0.041 & 0.889 $\pm$ 0.214 & 0.968 $\pm$ 0.046 \\
AGTIC - t1s1 & \textbf{1.000} $\pm$ \textbf{0.000} & \textbf{1.000} $\pm$ \textbf{0.000} & 0.900 $\pm$ 0.300 & \textbf{1.000} $\pm$ \textbf{0.000} & 0.850 $\pm$ 0.357 & 0.950 $\pm$ 0.071 \\
AGTIC - t3s2 & \textbf{1.000} $\pm$ \textbf{0.000} & \textbf{1.000} $\pm$ \textbf{0.000} & 0.950 $\pm$ 0.218 & 0.932 $\pm$ 0.215 & 0.825 $\pm$ 0.285 & 0.941 $\pm$ 0.072 \\
AGTIC - t2s1 & \textbf{1.000} $\pm$ \textbf{0.000} & \textbf{1.000} $\pm$ \textbf{0.000} & 0.900 $\pm$ 0.300 & \textbf{1.000} $\pm$ \textbf{0.000} & 0.800 $\pm$ 0.400 & 0.940 $\pm$ 0.089 \\
Hoeffding's D & \textbf{1.000} $\pm$ \textbf{0.000} & \textbf{1.000} $\pm$ \textbf{0.000} & \textbf{1.000} $\pm$ \textbf{0.000} & \textbf{1.000} $\pm$ \textbf{0.000} & 0.550 $\pm$ 0.497 & 0.910 $\pm$ 0.201 \\
MIC & 0.984 $\pm$ 0.015 & 0.977 $\pm$ 0.022 & 0.956 $\pm$ 0.160 & 0.891 $\pm$ 0.292 & 0.733 $\pm$ 0.422 & 0.908 $\pm$ 0.105 \\
AGTIC - t2s2 & \textbf{1.000} $\pm$ \textbf{0.000} & \textbf{1.000} $\pm$ \textbf{0.000} & 0.911 $\pm$ 0.269 & 0.905 $\pm$ 0.280 & 0.724 $\pm$ 0.341 & 0.908 $\pm$ 0.113 \\
AGTIC - t2s3 & \textbf{1.000} $\pm$ \textbf{0.000} & 0.997 $\pm$ 0.013 & 0.917 $\pm$ 0.235 & 0.950 $\pm$ 0.162 & 0.639 $\pm$ 0.338 & 0.900 $\pm$ 0.150 \\
AGTIC - t1s2 & \textbf{1.000} $\pm$ \textbf{0.000} & \textbf{1.000} $\pm$ \textbf{0.000} & 0.900 $\pm$ 0.300 & 0.834 $\pm$ 0.336 & 0.766 $\pm$ 0.324 & 0.900 $\pm$ 0.103 \\
AGTIC - t1s3 & \textbf{1.000} $\pm$ \textbf{0.000} & 0.999 $\pm$ 0.003 & 0.901 $\pm$ 0.286 & 0.870 $\pm$ 0.273 & 0.697 $\pm$ 0.301 & 0.893 $\pm$ 0.124 \\
dCor & \textbf{1.000} $\pm$ \textbf{0.000} & \textbf{1.000} $\pm$ \textbf{0.000} & 0.900 $\pm$ 0.300 & 0.800 $\pm$ 0.400 & 0.600 $\pm$ 0.490 & 0.860 $\pm$ 0.167 \\
HSIC & \textbf{1.000} $\pm$ \textbf{0.000} & \textbf{1.000} $\pm$ \textbf{0.000} & 0.850 $\pm$ 0.357 & 0.950 $\pm$ 0.218 & 0.500 $\pm$ 0.500 & 0.860 $\pm$ 0.210 \\
MI (k=20) & 0.945 $\pm$ 0.217 & 0.941 $\pm$ 0.216 & 0.795 $\pm$ 0.396 & 0.895 $\pm$ 0.299 & 0.692 $\pm$ 0.453 & 0.853 $\pm$ 0.108 \\
rdmCor & \textbf{1.000} $\pm$ \textbf{0.000} & \textbf{1.000} $\pm$ \textbf{0.000} & 0.300 $\pm$ 0.458 & 0.000 $\pm$ 0.000 & 0.200 $\pm$ 0.400 & 0.500 $\pm$ 0.469 \\
R$^2$ & \textbf{1.000} $\pm$ \textbf{0.000} & 0.350 $\pm$ 0.477 & 0.000 $\pm$ 0.000 & 0.000 $\pm$ 0.000 & 0.150 $\pm$ 0.357 & 0.300 $\pm$ 0.417 \\                 
            \hline
		\end{tabular}
	}  
\label{table:1DpowerObs}
\end{table*}

\begin{table*}[tbh]
	\centering
	\caption{Optimal thresholds of AGTIC across noise amplitudes over different relationships}
	\resizebox{1\linewidth}{!}{
		\begin{tabular}{ l | l | l | l | l | l }
			 &  linear & parabolic & sinusoidal & circular & checkerboard \\ \thickhline
            AGTIC - t0s1 & 0.466 & 0.403 & 0.460 & 0.632 & 0.532 \\
            AGTIC - t0s2 & l: 0.375, u: 0.649 & l: 0.323, u: 0.566 & l: 0.349, u: 0.626 & l: 0.332, u: 0.614 & l: 0.366, u: 0.652   \\
            AGTIC - t1s1 & l: 0.128, u: 0.624 & l: 0.092, u: 0.588 & l: 0.148, u: 0.624 & l: 0.120, u: 0.584 & l: 0.132, u: 0.556      \\
            AGTIC - t1s2 & l: 0.400, u: 0.804 & l: 0.372, u: 0.800 & l: 0.420, u: 0.796 & l: 0.360, u: 0.732 & l: 0.356, u: 0.740    \\
            AGTIC - t1s3 & l: 0.356, u: 0.752 & l: 0.316, u: 0.708 & l: 0.296, u: 0.712 & l: 0.336, u: 0.728 & l: 0.344, u: 0.736     \\
            AGTIC - t2s1 & l: 0.040, u: 0.648 & l: 0.024, u: 0.680 & l: 0.100, u: 0.648 & l: 0.024, u: 0.620 & l: 0.128, u: 0.640      \\
            AGTIC - t2s2 & l: 0.356, u: 0.776 & l: 0.368, u: 0.780 & l: 0.344, u: 0.748 & l: 0.336, u: 0.760 & l: 0.368, u: 0.804     \\
            AGTIC - t2s3 & l: 0.284, u: 0.704 & l: 0.252, u: 0.688 & l: 0.244, u: 0.684 & l: 0.264, u: 0.776 & l: 0.328, u: 0.756      \\
            AGTIC - t3s1 & l: 0.000, u: 0.228 & l: 0.024, u: 0.232 & l: 0.000, u: 0.228 & l: 0.000, u: 0.224 & l: 0.000, u: 0.204      \\
            AGTIC - t3s2 & l: 0.356, u: 0.836 & l: 0.364, u: 0.752 & l: 0.356, u: 0.812 & l: 0.368, u: 0.800 & l: 0.392, u: 0.812     \\
            AGTIC - t3s3 & l: 0.312, u: 0.720 & l: 0.288, u: 0.692 & l: 0.260, u: 0.720 & l: 0.328, u: 0.760 & l: 0.332, u: 0.682     \\
            pAGTIC - t1s1 & l: 0.000, u: 0.380 & l: 0.000, u: 0.424 & l: 0.020, u: 0.404 & l: 0.024, u: 0.348 & l: 0.012, u: 0.332       \\
            pAGTIC - t1s2 & l: 0.388, u: 0.756 & l: 0.372, u: 0.788 & l: 0.448, u: 0.816 & l: 0.352, u: 0.760 & l: 0.364, u: 0.748      \\
            pAGTIC - t1s3 & l: 0.216, u: 0.668 & l: 0.280, u: 0.704 & l: 0.368, u: 0.728 & l: 0.232, u: 0.652 & l: 0.324, u: 0.728      \\
            pAGTIC - t2s1 & l: 0.032, u: 0.632 & l: 0.000, u: 0.600 & l: 0.016, u: 0.576 & l: 0.004, u: 0.600 & l: 0.012, u: 0.620  \\
            pAGTIC - t2s2 & l: 0.424, u: 0.820 & l: 0.408, u: 0.807 & l: 0.360, u: 0.796 & l: 0.332, u: 0.788 & l: 0.364, u: 0.840 \\
            pAGTIC - t2s3 & l: 0.280, u: 0.756 & l: 0.252, u: 0.704 & l: 0.304, u: 0.684 & l: 0.296, u: 0.724 & l: 0.276, u: 0.696 \\
            pAGTIC - t3s1 & l: 0.000, u: 0.200 & l: 0.000, u: 0.200 & l: 0.000, u: 0.200 & l: 0.000, u: 0.200 & l: 0.000, u: 0.200 \\
            pAGTIC - t3s2 & l: 0.464, u: 0.864 & l: 0.516, u: 0.888 & l: 0.428, u: 0.872 & l: 0.436, u: 0.816 & l: 0.412, u: 0.816 \\
            pAGTIC - t3s3 & l: 0.344, u: 0.764 & l: 0.356, u: 0.740 & l: 0.352, u: 0.716 & l: 0.256, u: 0.660 & l: 0.280, u: 0.672 \\
            \hline
		\end{tabular}
	}  
\label{table:optThresh}
\end{table*}

\clearpage
\section{Supplementary Proof for Theorem 4.1}
\label{sec:indeptheory}

This section aims to be prove that the distance correlation computed from GT-transformed distances still satisfies the independence criterion. As the fundamental building blocks of AGTIC, distance correlation \cite{szekely2007measuring,szekely2009brownian} generalizes the idea of correlation such that:  
	\begin{itemize}
	\item $\mathcal{R}(X,Y)$ is defined for $X, Y$ in arbitrary dimensions.
    \item $\mathcal{R}(X,Y)=0 \Leftrightarrow X$ and $Y$ are independent.  
	\end{itemize}

Here we wish to extend the proof of these two properties to our adaptive approach, which involves three piecewise functions of linear transformation. Consider a \textit{Cauchy sequence} $\{f_n\}^\infty_{n=1}$ in a normed vector space S, it satisfies the existence of an integer for any $\epsilon$ such that, $\norm{f_n - f_m} < \epsilon$ for all $n < N$, $m > N$. The normed vector space $S$ is said to be \textit{complete} if every \textit{Cauchy sequence} converges to a limit in $S$ as a \textit{Banach space}. Here we define $f(X-X')$ as a simplified the geo-topological transform to be a continuous non-linear \textit{bounded} \textit{functional} onto $L^2[0,1]$:

\begin{equation}
f_n(t) =
  \begin{cases}
0 & \text{if $0 \leq t \leq \frac{1}{2} - \frac{1}{n}$} \\
\frac{1}{2}+ \frac{n}{2}(t-\frac{1}{2}) & \text{if $\frac{1}{2} - \frac{1}{n} \leq t \leq \frac{1}{2} + \frac{1}{n}$} \\
1 & \text{if $\frac{1}{2} + \frac{1}{n} \leq t \leq 1$}
  \end{cases}
\end{equation}

where the upper and lower thresholds defined in Figure \ref{fig:RGTA} are replaced with one parameter $n$ (but the theoretical proof is equivalent in the two thresholds case with a linear transformation). From its graph, we see it ``converges'' to:

\begin{equation}
\label{eq:functional}
f_0(t) =
  \begin{cases}
0 & \text{if $0 \leq t \leq \frac{1}{2}$} \\
\frac{1}{2} & \text{if $t = \frac{1}{2}$} \\
1 & \text{if $\frac{1}{2} \leq t \leq 1$}
  \end{cases}
\end{equation}

So we can calculate the expected value for the Euclidean distance (within $L^2[0,1]$) after transformation as:

\begin{equation}
\label{eq:EXGT}
\begin{split}
\EX[f(x)] & = \int_0^1 x f(x)dx \\
& = 0+ \int_{\frac{1}{2} - \frac{1}{n}}^{\frac{1}{2} + \frac{1}{n}} x (\frac{nx}{2}+ \frac{1-\frac{n}{2}}{2})dx + \int_{\frac{1}{2} + \frac{1}{n}}^1 x dx \\
&=\frac{n}{3x^3}+\frac{n}{4x}-\frac{1}{2n^2}-\frac{1}{4n}+\frac{1}{4}
\end{split}
\end{equation}

which is a monotone nonlinear transformation. In another word, we wish to apply this \textit{monotone nonlinear operator} in a \textit{Hilbert Space} (the original distance correlation). Here we define ${\mathcal{V}}^{2*}(X,Y)$ as the dCor calculated given a distance matrix after the proposed geo-topological transformation. If $X$ and $Y$ are independent, then $\EX[XY] = \EX[X]\EX[Y]$

\begin{equation}
\label{eq:dCorGT}
\begin{split}
{\mathcal{V}}^{2*}(X,Y) = & \EX[f(\norm{X-X'}_2)f(\norm{f(Y-Y'}_2)] \\
& + \EX [f(\norm{X-X'}_2)] \EX [f(\norm{Y-Y'}_2)] \\
& - \EX [f(\norm{X-X''}_2) f(\norm{Y-Y'}_2) ] \\
& - \EX [f(\norm{X-X'}_2) f(\norm{Y-Y''}_2) ] \\
= & 2\EX[f(\norm{X-X'}_2) f(\norm{f(Y-Y'}_2)] \\
& - 2\EX [f(\norm{X-X'}_2) f(\norm{Y-Y'}_2) ] \\
= & 0 \\
\end{split}
\end{equation}

Thus, the independence criterion still holds regardless of the applied \textit{functional}. Then we are going to look at the threshold searching process to determine whether the independence criterion still holds for $\mathcal{V}^{2*}(X,Y)$. We here denote ${\mathcal{V}^*}^2_{max}(X,Y)$ and define it to be ${\mathcal{V}}^{2*}_{max}(X,Y) = \sup_{n\in\mathbb{R}} {\mathcal{V}}^{2*}(X,Y)$. 

The equation \ref{eq:functional} is Cauchy, but the convergence is not in the $\sup$ norm: 

\begin{equation}
\label{eq:supConvergence1}
\begin{split}
\norm{f_n-f_0} = & \sup_{t\in[0,1]} |f_n(t)-f_0(t)| = 1/2 \\
\end{split}
\end{equation}

which doesn't converge to zero, but instead, it converges to $f_0$ in the $L^2$ norm as follows:

\begin{equation}
\label{eq:supConvergence2}
\begin{split}
\norm{f_n-f_0}_2 = & \sqrt[]{\int_0^1(f_n(t)-f_0(t))^2dt} \\
\leq & \sqrt[]{\int_{\frac{1}{2}-\frac{1}{n}}^{\frac{1}{2}+\frac{1}{n}}(1)^2dt} = \sqrt[]{\frac{2}{n}} \\
\end{split}
\end{equation}

which converges to zero, showing that the space $L^2[a,b]$ as the completion of all monotone continuous functions (including $GT(\cdot)$) on $[a,b]$ in the $L^2$ norm. Therefore, ${\mathcal{V}}^{2*}(X,Y)$ still holds the independence criterion.

\end{document}